\documentclass[journal]{IEEEtran}
\usepackage[T1]{fontenc}
\usepackage{amsmath,amssymb,amsfonts,mathtools}
\usepackage{amsthm}
\usepackage{graphicx}
\usepackage{booktabs}
\usepackage{multirow}
\usepackage{algorithm}
\usepackage[noend]{algorithmic}
\usepackage{url}
\usepackage{xcolor}
\usepackage{enumitem}
\usepackage{microtype}

\graphicspath{{figures/}}

\newtheorem{theorem}{Theorem}
\newtheorem{proposition}{Proposition}
\newtheorem{lemma}{Lemma}
\newtheorem{corollary}{Corollary}
\newtheorem{definition}{Definition}
\newtheorem{assumption}{Assumption}
\newtheorem{remark}{Remark}

\newcommand{\E}{\mathbb{E}}
\newcommand{\Prob}{\mathbb{P}}
\newcommand{\R}{\mathbb{R}}
\newcommand{\auc}{\mathrm{AUC}}
\newcommand{\aucmax}{\mathrm{AUC}_{\max}}

\newcommand{\vol}{\mathrm{vol}}
\newcommand{\tr}{\mathrm{tr}}
\newcommand{\dB}{d_B}
\newcommand{\betaB}{\beta_B}
\newcommand{\kap}{\kappa}
\newcommand{\astar}{\alpha^{*}}
\newcommand{\yes}{\checkmark}

\begin{document}

\title{When and Why Do Linear Bias Probes Fail?\\
A Geometric and Statistical Theory of Bias Detectability\\
in Large Language Model Representations}

\author{Mo Hai, Haifeng Li}


\maketitle

\begin{abstract}
Linear probing is the standard instrument for detecting social biases in the hidden representations of large language models. Yet reported probe accuracies come almost exclusively from \emph{counterfactual} evaluations in which every input carries an explicit demographic marker---a regime that turns out to be trivially easy (AUC $\approx1.0$ throughout our experiments). Once only a fraction $\alpha$ of inputs carries demographic information, performance degrades sharply, and a weak probe may reflect either an unbiased model or an underpowered detector. We develop a theory that resolves this ambiguity. Modeling representations as two class-conditional clusters with Mahalanobis separation $s$ on a manifold of curvature $\kap$, we prove: (i) a finite-sample generalization bound governed by the manifold's extrinsic radius---capped by positive curvature (Bonnet--Myers)---with a matching $\smash{\sqrt{\dB/n}}$ minimax lower bound; (ii) an exact purity law for the maximum linear-probe AUC, strictly increasing in $\alpha$; (iii) a \emph{curvature ceiling}: ambient chordal separation on a space form cannot exceed $2/\sqrt{\kap}$; and (iv) a \emph{detectability threshold} $\astar(n)\approx 2.05\,(z_{1-\delta}+z_{1-\gamma})/(s\sqrt{n})$ below which no audit can distinguish probe output from chance. Every theorem is validated on synthetic manifolds with known ground truth and on six open-weight models $\times$ four bias dimensions, where the purity law predicts entire AUC--$\alpha$ curves from a single cross-fitted $\hat s$ measured at $\alpha=1$, with no parameters fitted to those curves. Three further checks close the loop: the detection test is calibrated on exact nulls that replicate the full protocol; kernel and MLP baselines confirm that the curvature ceiling is specific to linear readout, whereas the purity law---a Bayes bound---caps nonlinear probes as well; and stereotype leakage in ``neutral'' text is measured against occupation statistics rather than assumed absent. The framework turns bias auditing into a power analysis: given a target purity and effect size, it prescribes the sample budget $n(\alpha)$ for a conclusive audit.
\end{abstract}

\begin{IEEEkeywords}
Bias probing, large language models, linear probes, information geometry, generalization bounds, detectability, representation learning, fairness auditing.
\end{IEEEkeywords}

\section{Introduction}
\IEEEPARstart{C}{onsider} an auditor tasked with certifying that a language model does not encode gender stereotypes. Following standard practice, she trains a logistic-regression probe on the model's hidden states. On a benchmark of minimal counterfactual pairs (``\emph{He} is a doctor'' vs.\ ``\emph{She} is a doctor''), the probe separates the two classes perfectly: AUC $=1.0$. On a realistic corpus, in which roughly seven of every ten sentences mention no demographic attribute at all, the same probe on the same model returns an AUC of $0.70$---and on a smaller audit sample of a hundred sentences with sparser markers, $0.55$, statistically indistinguishable from coin flipping. Which number should the audit report? Is the model saturated with gender information, moderately biased, or clean? Is the third probe telling us the bias is absent---or merely that the probe is blind?

This paper develops the theory needed to answer such questions. Its central claim is that linear-probe performance is not an intrinsic property of a model: it is a joint function of the \emph{representation geometry} (class separation $s$, manifold curvature $\kap$, intrinsic dimension $\dB$), the \emph{signal purity} $\alpha$ (the fraction of inputs that carry demographic evidence), and the \emph{sample size} $n$. All three appear in closed form in our results, and all three are measurable. Once they are measured, the probe's behavior across regimes---including its apparent failures---becomes predictable, in several cases with no free parameters.

The need for such a theory is practical. Linear probes underpin bias benchmarks, model cards, and debiasing pipelines: null-space projection and concept-erasure methods explicitly assume that bias occupies a linear subspace \cite{ravfogel2020inlp,ravfogel2022rlace,belrose2023leace}, and intervention techniques steer behavior along probe directions \cite{li2023iti}. Yet the field has repeatedly documented probe pathologies---probes that memorize dataset artifacts \cite{hewitt2019control}, probes whose accuracy does not transfer to downstream harms \cite{goldfarb2021intrinsic,delobelle2022measuring}, and debiasing that hides rather than removes information \cite{gonen2019lipstick}. What has been missing is a quantitative account of \emph{when} a linear probe is a reliable instrument and when it is underpowered, and a way to tell an unbiased model from an underpowered audit.

\subsection{What the theory says}
We model the probing task generatively. A fraction $\alpha$ of texts carry a demographic marker and induce representations from one of two clusters whose centers are separated by Mahalanobis distance $s$; the remaining $1-\alpha$ are neutral and carry no label information. The clusters live on (a neighborhood of) a submanifold whose sectional curvature is bounded by $\kap$. Four results organize the paper.

\emph{(1) Sample complexity (Theorem~\ref{thm:gen}, Proposition~\ref{prop:lower}).} The generalization gap of norm-bounded linear probes decays as $n^{-1/2}$, with a complexity constant governed by the bias submanifold's extrinsic radius---a quantity that a positive lower curvature bound provably caps, via the Bonnet--Myers diameter theorem---and with a matching minimax lower bound $\Omega(\sqrt{\dB/n})$ in the intrinsic dimension. The practical content is a budget rule: halving the audit's uncertainty requires quadrupling $n$. On space forms the cap tightens to the sphere radius, and there curvature helps generalization exactly where it hurts separability (Theorem~\ref{thm:curvature}).

\emph{(2) Purity (Theorem~\ref{thm:mixture}).} The best achievable AUC at purity $\alpha$ is exactly
$\alpha^{2}\Phi(\frac{s}{\sqrt 2})+2\alpha(1-\alpha)\Phi(\frac{s}{2\sqrt 2})+\frac{(1-\alpha)^{2}}{2}$.
This function is strictly increasing in $\alpha$ and saturates at the binormal value $\Phi(s/\sqrt2)$. Two consequences matter. First, the counterfactual regime ($\alpha=1$) with the large separations typical of surface demographic markers ($\hat s\approx 8$--$40$ in our measurements) sits deep in the saturated zone---AUC $\approx 1.0$ conveys almost no information about $s$ beyond a loose lower bound, which is why counterfactual benchmarks overstate real-world separability. Second, monotonicity is a falsifiable prediction: no ``inversion'' or interior optimum of AUC as a function of purity is possible in this model class, and none is observed in our sweeps.

\emph{(3) Curvature (Theorem~\ref{thm:curvature}).} Linear probes operate on ambient coordinates, hence on \emph{chords} of the representation manifold. On a manifold of curvature $\kap>0$, the chord subtended by geodesic distance $r$ has length $\frac{2}{\sqrt\kap}\sin(\frac{\sqrt\kap\,r}{2})\le\min(r,\,2/\sqrt\kap)$: ambient separation saturates at the manifold diameter scale $2/\sqrt{\kap}$ no matter how far apart the classes are geodesically. The resulting AUC ceiling $\Phi\big(\sqrt{2}/(\sqrt{\kap}\,\sigma_{w})\big)$ decreases monotonically from $1$ (flat) to $1/2$ ($\kap\to\infty$), giving the correct, legal-range replacement for curvature-based failure prediction, and quantifying what a geodesic-aware (nonlinear) probe could recover that a linear one cannot---a prediction E10 tests directly with kernel and MLP probes.

\emph{(4) Detectability (Theorem~\ref{thm:detect}).} An audit is a hypothesis test. Because the null distribution of empirical AUC on a held-out sample is known exactly (a $U$-statistic with variance $\frac{n_++n_-+1}{12 n_+ n_-}$; Section~\ref{sec:stats} maps the working CV statistic onto this null and calibrates it empirically), there is a critical purity
$\astar(n) \approx 2.05\,(z_{1-\delta}+z_{1-\gamma})/(s\sqrt n)$
below which no level-$\delta$ test achieves power $1-\gamma$. The $(\alpha,n)$ plane splits into a detectable and an undetectable phase with boundary slope $-\tfrac12$ on log--log axes. In the undetectable phase a null probe result is \emph{uninformative about the model}---the honest report is ``underpowered,'' not ``unbiased.''

\subsection{What the experiments show}
We validate on two levels. On synthetic manifolds where $s$, $\kap$, $\dB$, and $\alpha$ are known exactly, the closed forms hold to within $0.006$--$0.013$ AUC (oracle scoring), estimators of $s$, $\kap$, and $\dB$ are calibrated against ground truth, and the empirical detection boundary tracks $\astar(n)$. On six open-weight models across four bias dimensions, a single cross-fitted $\hat s$ measured on counterfactual pairs predicts the entire AUC--$\alpha$ curve with no further tuning; the $(\alpha,n)$ detection phase diagram matches the predicted boundary; the generalization gap follows the $n^{-1/2}$ law; and layer-wise scans, leakage-controlled splits, and LLM-paraphrased naturalistic prompts probe the robustness of each conclusion. A third, adversarial tier targets the protocol itself: E9 measures the realized false-positive rate of the detection test on exact nulls that replicate the full pipeline, E10 adds kernel and MLP probes that separate the linear-specific curvature ceiling from the Bayes-limited purity law, and E11 measures---rather than assumes away---the stereotype leakage in nominally neutral text. All numbers in every table are generated by the released scripts directly from logged experiment outputs; each table caption names its source file.

\subsection{Contributions}
\begin{itemize}[leftmargin=1.2em]
\item A generative model of bias probing with three measurable control variables ($s$, $\kap$--$\dB$ geometry, $\alpha$, $n$), under which we prove: a curvature-corrected generalization bound with matching minimax lower bound; an exact purity--AUC law; a curvature ceiling for ambient linear separation; and a sharp detectability threshold $\astar(n)$ (Section~\ref{sec:theory}).
\item A validated measurement pipeline: cross-fitted Mahalanobis separation (removing the $O(d/n)$ plug-in inflation), a label-free concentration index $\betaB\in[0,1-1/d]$ with a proven link to separation and curvature, and estimators validated on ground-truth manifolds before being applied to LLMs (Sections~\ref{sec:exp-design}--\ref{sec:results}).
\item A two-tier validation---synthetic with known truth, then six LLMs $\times$ four bias dimensions---in which the theory's zero-free-parameter predictions are tested and their errors reported as measured (Section~\ref{sec:results}).
\item A reframing of bias auditing as power analysis, with explicit sample budgets $n(\alpha)$ and an ``underpowered vs.\ unbiased'' decision rule for null results (Section~\ref{sec:discussion}).
\end{itemize}

\section{Related Work}\label{sec:related}

\subsection{Bias measurement in NLP representations}
Bolukbasi et al.~\cite{bolukbasi2016man} showed that gender associations in static word embeddings admit approximately linear structure, and Caliskan et al.~\cite{caliskan2017semantics} quantified human-like associations with WEAT. Sentence- and context-level successors include SEAT \cite{may2019seat}, StereoSet \cite{nadeem2021stereoset}, CrowS-Pairs \cite{nangia2020crows}, BBQ \cite{parrish2022bbq}, and coreference-based tests \cite{zhao2018winobias,rudinger2018winogender}. A parallel critical literature documents that these instruments disagree with one another and with downstream harms \cite{goldfarb2021intrinsic,delobelle2022measuring,blodgett2020language}, and that linear debiasing can mask rather than remove information \cite{gonen2019lipstick}. Our contribution is orthogonal to benchmark design: we treat any marker-based bias test as a detection problem and characterize its statistical power. The counterfactual/realistic gap we quantify (AUC $1.0$ vs.\ $0.55$--$0.75$) gives one mechanism for the poor correlation between intrinsic metrics and deployment behavior: intrinsic tests are usually run at $\alpha=1$, deployment corpora live at small $\alpha$.

\subsection{Probing methodology and its critiques}
Linear probes were introduced as diagnostic classifiers \cite{alain2016understanding,conneau2018cram} and remain the default tool for representation analysis \cite{tenney2019bert,liu2019linguistic,hewitt2019structural,belinkov2022probing}. Methodological critiques motivated control tasks and selectivity \cite{hewitt2019control}, information-theoretic probes \cite{pimentel2020information}, minimum-description-length probes \cite{voita2020mdl}, and behavioral tests of whether probed information is \emph{used} \cite{elazar2021amnesic}. Concept-removal methods---INLP \cite{ravfogel2020inlp}, R-LACE \cite{ravfogel2022rlace}, LEACE \cite{belrose2023leace}---assume linear geometry, as do inference-time interventions \cite{li2023iti}, representation-engineering pipelines that extract and steer concept directions at scale \cite{zou2023representation}, single-direction behavioural mediation \cite{arditi2024refusal}, and sparse-autoencoder feature dictionaries \cite{cunningham2024sparse}; our power analysis applies to any such linear detector. Surveys of LLM bias instruments \cite{gallegos2024bias} catalogue the tests whose statistical power we characterize. We add the missing power analysis: control tasks ask whether a probe's \emph{positive} finding is trustworthy; our detectability threshold determines when a probe's \emph{null} finding is meaningful. Our grouped, template-disjoint cross-validation follows the leakage concerns of \cite{hewitt2019control}.

\subsection{Geometry of neural representations}
The linear-representation hypothesis holds that high-level concepts occupy linear directions \cite{mikolov2013efficient,arora2016latent,park2023linear}, with supporting evidence for truth \cite{marks2023geometry}, sentiment \cite{tigges2023linear}, space and time \cite{gurnee2024language}, and board-game state \cite{nanda2023emergent}. Complementary work measures the nonlinear side: anisotropy \cite{ethayarajh2019contextual,mimno2017strange}, intrinsic dimension \cite{ansuini2019intrinsic,aghajanyan2020intrinsic,cai2023intrinsic,valeriani2023geometry}, and visualizable manifold structure \cite{reif2019visualizing}. Manifold assumptions in learning theory go back to \cite{belkin2003laplacian,narayanan2010sample,fefferman2016testing}. We connect these threads to bias probing through comparison geometry: curvature bounds control both covering numbers (hence sample complexity, via Günther's volume comparison \cite{gallot2004riemannian,docarmo1992riemannian}) and the chord-versus-geodesic contraction that limits ambient linear separability.

\subsection{Information geometry and statistical foundations}
Information geometry equips statistical models with the Fisher metric \cite{amari2000methods,amari2016information,cencov1982statistical,nielsen2020elementary}; applications to deep learning include natural gradient \cite{amari1998natural}, Fisher spectra \cite{karakida2019universal}, and capacity measures \cite{liang2019fisher}. Our generalization analysis uses Rademacher complexity and chaining \cite{bartlett2002rademacher,dudley1967sizes,wainwright2019high,vershynin2018high}; the lower bound uses Assouad's method \cite{tsybakov2009introduction}. The AUC analysis rests on the binormal ROC model \cite{hanley1982meaning} and Fisher's linear discriminant \cite{fisher1936use} and on exact null properties of the Mann--Whitney statistic \cite{mann1947test,bamber1975area,delong1988comparing}. Covariance estimation in the $d\gtrsim n$ regime uses shrinkage \cite{ledoit2004well,chen2010shrinkage}; intrinsic dimension estimation uses maximum likelihood \cite{levina2004maximum,mackay2005comments} with the two-NN estimator \cite{facco2017estimating} as a cross-check. Relative to this toolbox our novelty is the assembly: an exact, checkable pipeline from measurable geometry to probe power.

\section{Problem Setup and Preliminaries}\label{sec:setup}

\begin{figure}[!t]
\centering
\includegraphics[width=\columnwidth]{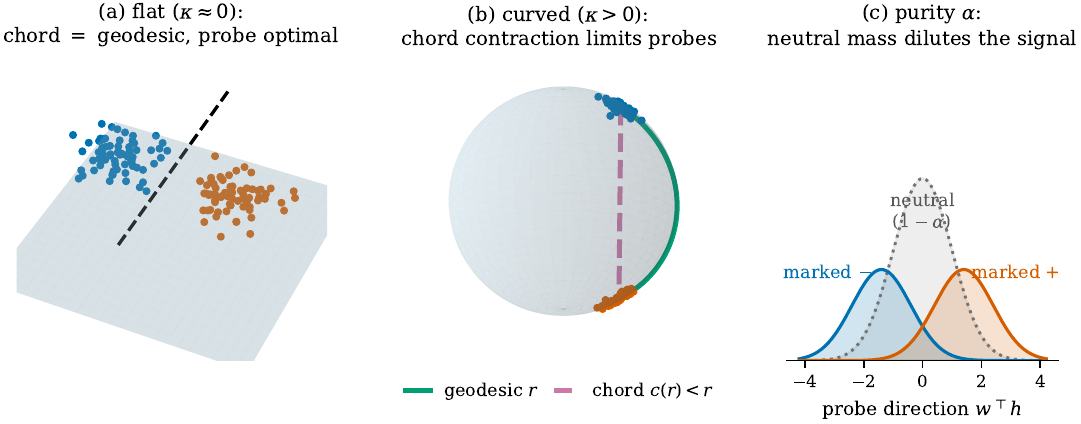}
\caption{Geometric picture. (a)~On a flat bias submanifold, chords equal geodesics and linear probes are optimal (Lemma~\ref{lem:binormal}). (b)~Positive curvature contracts chords relative to geodesics; ambient separation saturates at the diameter scale $2/\sqrt{\kap}$ (Theorem~\ref{thm:curvature}). (c)~Signal purity $\alpha$: only a fraction of samples carry the marker; neutral texts contribute label-independent mass that dilutes the probe signal (Theorem~\ref{thm:mixture}).}
\label{fig:concept}
\end{figure}

\subsection{Probing task}
Let $f_{\ell}:\mathcal{X}\to\R^{d}$ be the layer-$\ell$ representation map of a language model (token-averaged hidden states in our experiments). A \emph{linear probe} is $g_{w,b}(h)=\mathbf{1}[w^{\top}h+b>0]$, trained on $\{(f_{\ell}(x_i),y_i)\}_{i=1}^{n}$ where $y_i\in\{0,1\}$ encodes a demographic attribute of the text-generating process. Performance is the area under the ROC curve, $\auc=\Prob(S_{+}>S_{-})$, for the probe score $S=w^{\top}h$ on held-out data. We write $\Phi$ and $\phi$ for the standard normal CDF and density, and $z_{q}=\Phi^{-1}(q)$.

\subsection{Generative model}\label{sec:genmodel}
\begin{assumption}[Mixed-purity cluster model]\label{ass:model}
A sample with label $y\in\{\pm\}$ is \emph{marked} with probability $\alpha$ and \emph{neutral} with probability $1-\alpha$, independently of $y$. Marked representations are drawn from $\mathcal{N}(\mu_{y},\Sigma)$ with $\Delta=\mu_{+}-\mu_{-}$ and Mahalanobis separation $s^{2}=\Delta^{\top}\Sigma^{-1}\Delta$; neutral representations are drawn from $\mathcal{N}(\mu_{0},\Sigma)$ with $\mu_{0}=\tfrac12(\mu_{+}+\mu_{-})$, independent of $y$.
\end{assumption}
The neutral-midpoint condition is tested empirically in Section~\ref{sec:results-e1} (measured standardized offsets $b_w$ are reported per model; Corollary~\ref{cor:offset} gives the formula for $b_w\ne0$). The Gaussian assumption is a local model of a cluster pair on the representation manifold; Section~\ref{sec:theory-curv} adds the manifold's curvature explicitly, and the synthetic experiments test both layers of the model separately.

\subsection{Bias submanifold and geometry}
We model the marked representations as concentrated near a compact $\dB$-dimensional Riemannian submanifold $B\subset\R^{d}$ (the \emph{bias submanifold}), with sectional curvature bounded, $|K|\le\kap_{\max}$, diameter $D_B$, and Riemannian volume $\vol(B)$. Three measurable summaries appear in our results:
the cross-fitted Mahalanobis separation $\hat s$ (Section~\ref{sec:estimators}), the intrinsic dimension $\hat\dB$ (Levina--Bickel MLE \cite{levina2004maximum} with the correction of \cite{mackay2005comments}), and the concentration index $\betaB$ (Definition~\ref{def:beta}).

\subsection{Rademacher complexity}
For a class $\mathcal F$ of real functions and i.i.d.\ sample $S$, the empirical Rademacher complexity is $\widehat{\mathfrak R}_{S}(\mathcal F)=\E_{\sigma}\sup_{f\in\mathcal F}\frac1n\sum_i\sigma_i f(x_i)$ with independent signs $\sigma_i$. For any $\delta>0$, with probability $\ge1-\delta$, every $f$ in a $[0,1]$-loss class satisfies $R(f)\le\widehat R_n(f)+2\mathfrak R_n(\mathcal F)+\sqrt{\log(1/\delta)/2n}$ \cite{bartlett2002rademacher}.

\subsection{Notation}
Table~\ref{tab:notation} collects the recurring symbols, and for each quantity that must be measured on data it names the estimator and the experiment in which that estimator is calibrated against ground truth before use.

\begin{table}[!t]
\centering
\caption{Notation. Quantities marked $\bullet$ are estimated from data; their estimators are calibrated on ground-truth manifolds in E3-D before being read on LLM features.}
\label{tab:notation}
\scalebox{0.9}{
\begin{tabular}{@{}lll@{}}
\toprule
Symbol & Meaning & Source / estimator \\
\midrule
$\alpha$ & signal purity (marked fraction) & set by protocol (\S\ref{sec:prompts}) \\
$s$ $\bullet$ & Mahalanobis class separation & cross-fitted $\hat s$ (\S\ref{sec:estimators}) \\
$\kap$ $\bullet$ & sectional curvature & sagitta $\hat\kap$ (\S\ref{sec:estimators}) \\
$\dB$ $\bullet$ & intrinsic dimension of $B$ & MLE $\hat\dB$; Two-NN check \\
$\betaB$ & concentration index & spectrum of pooled features \\
$n$, $n_{\mathrm{te}}$ & train / held-out sample sizes & set by protocol \\
$n_{+},n_{-}$ & class counts in the test set & $n_{+}{+}n_{-}=n_{\mathrm{te}}$ \\
$\sigma_{0}$ & null s.d.\ of $\widehat{\auc}$ & exact, Theorem~\ref{thm:detect} \\
$r$, $c(r)$, $g$ & geodesic dist., chord, ratio & model quantities (Thm.~\ref{thm:curvature}) \\
$\sigma_w$ & within-class s.d.\ along chord & known on synthetic; see \S\ref{sec:results-e3} \\
$\rho_B$, $D_B$ & extrinsic radius, geo.\ diameter & bounded via \eqref{eq:radius} \\
$b$ & standardized neutral offset & measured per pair (E1) \\
$\astar(n)$ & detectability threshold & closed form \eqref{eq:alphastar} \\
$\Phi,\phi,z_q$ & normal CDF/pdf/quantile & --- \\
\bottomrule
\end{tabular}
}
\end{table}

\section{Theory}\label{sec:theory}
All proofs are given in full in the supplementary appendices~\ref{app:gen}--\ref{app:beta}. Figure~\ref{fig:theory} plots the three laws.

\begin{figure*}[!t]
\centering
\includegraphics[width=\textwidth]{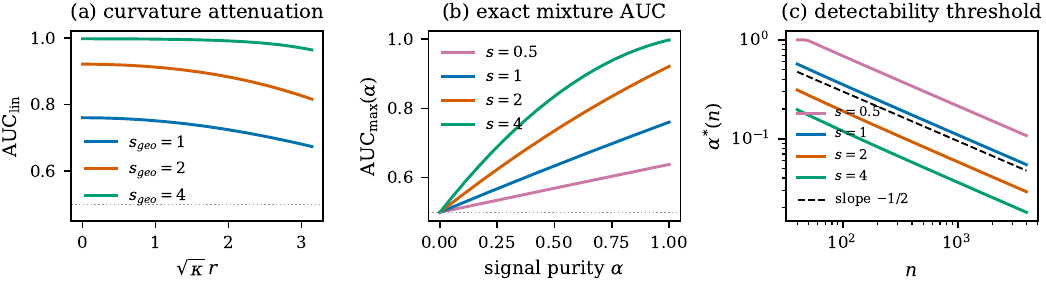}
\caption{The three laws of linear bias detectability. (a)~Curvature attenuation and ceiling (Theorem~\ref{thm:curvature}): ambient AUC as a function of $\sqrt{\kap}\,r$ for fixed geodesic SNR. (b)~Exact purity law (Theorem~\ref{thm:mixture}): $\aucmax(\alpha)$ is strictly increasing---no interior optimum is possible. (c)~Detectability threshold (Theorem~\ref{thm:detect}): $\astar(n)$ on log--log axes with the asymptotic slope $-1/2$.}
\label{fig:theory}
\end{figure*}

\subsection{The binormal baseline}
\begin{lemma}[Optimal linear AUC, flat case]\label{lem:binormal}
Under Assumption~\ref{ass:model} with $\alpha=1$, for any $w\neq0$ the probe score is Gaussian in each class and
\begin{equation}
\auc(w)=\Phi\!\Big(\tfrac{w^{\top}\Delta}{\sqrt{2\,w^{\top}\Sigma w}}\Big),
\qquad
\aucmax=\Phi\!\big(s/\sqrt{2}\big),
\label{eq:binormal}
\end{equation}
attained at $w^{*}\propto\Sigma^{-1}\Delta$. Moreover no measurable classifier exceeds $\Phi(s/\sqrt2)$: for equal-covariance Gaussians the likelihood ratio is monotone in $w^{*\top}h$, so the optimal linear probe is Bayes-optimal in ROC.
\end{lemma}
Lemma~\ref{lem:binormal} fixes units: all failure mechanisms below act by shrinking the argument of $\Phi$. It also yields a diagnostic used throughout: \emph{if measured AUC falls short of $\Phi(\hat s/\sqrt2)$ by more than sampling error, some assumption of the flat model (purity, curvature, or Gaussianity) is violated---and the gap is informative.}

\subsection{Sample complexity on a curved submanifold}\label{sec:theory-gen}
\begin{theorem}[Generalization bound with a curvature dividend]\label{thm:gen}
Let the marked features lie on a closed (compact, boundaryless) $\dB$-dimensional Riemannian submanifold $B\subset\R^{d}$ with geodesic diameter $D_{B}$ and sectional curvature $K\ge\kap_{\min}>0$, and let $\rho_{B}$ denote the radius of the smallest enclosing Euclidean ball of $B$. Let $\mathcal H_{\Lambda}=\{h\mapsto w^{\top}(h-c)\,{:}\;\|w\|_{2}\le\Lambda\}$ (centered at the enclosing-ball center $c$) and let $R$, $\widehat R_{n}$ denote true and empirical risk under the unit-margin ramp loss, which upper-bounds the $0$-$1$ risk. Then for any $\delta>0$, with probability $\ge1-\delta$, all $h\in\mathcal H_{\Lambda}$ satisfy
\begin{equation}
R_{0\text{-}1}(h)\;\le\;\widehat R_{n}(h)
\;+\;\frac{2\Lambda\,\rho_{B}}{\sqrt{n}}
\;+\;3\sqrt{\frac{\log(2/\delta)}{2n}},
\label{eq:genbound}
\end{equation}
and the ambient radius obeys the \emph{curvature ceiling}
\begin{equation}
\rho_{B}\;\le\;\tfrac{1}{\sqrt2}\,D_{B}^{\mathrm{ch}},
\qquad D_B^{\mathrm{ch}}\;\le\;D_B\;\le\;\frac{\pi}{\sqrt{\kap_{\min}}},
\label{eq:radius}
\end{equation}
where $D_{B}^{\mathrm{ch}}$ is the chordal diameter, the $1/\sqrt2$ factor is Jung's inequality (dimension-free form), the diameter bound is Bonnet--Myers, and the constant $\pi$ is sharp (Remark~\ref{rem:needle}).
\end{theorem}
The bound is \emph{dimension-free} in the ambient dimension: the geometry enters only through the extrinsic radius $\rho_B$, which a positive lower curvature bound caps at $\pi/(\sqrt2\sqrt{\kap_{\min}})$. On the round sphere of curvature $\kap$ (the space-form setting of Theorem~\ref{thm:curvature}) the cap sharpens to $\rho_B=1/\sqrt{\kap}$, and there curvature is a double-edged sword: it improves generalization (smaller $\rho_B$ in \eqref{eq:genbound}) exactly while it degrades separability (shorter chords in Theorem~\ref{thm:curvature}). Beyond space forms no such chordal contraction is available (Remark~\ref{rem:needle}), which is why \eqref{eq:radius} carries the Myers constant $\pi$ rather than the spherical $2$. The role of the intrinsic dimension $\dB$ is captured---provably---from below:

\begin{proposition}[Minimax lower bound]\label{prop:lower}
There exist absolute constants $c_{0},c_{1}>0$ such that for every $\dB\ge1$ and $n\ge c_{1}\dB$ and every learning rule $\mathcal A$ mapping samples to classifiers, there is a distribution consistent with Assumption~\ref{ass:model} ($\alpha=1$, separation $s\asymp\sqrt{\dB/n}\le1$, features supported on a $\dB$-dimensional affine slice, $\kap=0$) for which
\begin{equation}
\E\big[R_{0\text{-}1}(\mathcal A)-R^{*}\big]\;\ge\;c_{0}\,\min\Big\{1,\;\sqrt{\dB/n}\Big\}.
\label{eq:lower}
\end{equation}
\end{proposition}
Together: the $n^{-1/2}$ rate is correct and not improvable, the intrinsic dimension $\dB$ is the correct hardness parameter from below, and the upper bound's geometry dependence is through the extrinsic radius. Experiment E6 checks the $n^{-1/2}$ law directly (log--log slope of the measured gap).

\subsection{Curvature: the chordal ceiling}\label{sec:theory-curv}
Linear probes read \emph{ambient} coordinates. If the two marked clusters sit on a curved manifold at geodesic distance $r$, the probe only sees the chord.

\begin{theorem}[Curvature attenuation]\label{thm:curvature}
Let $M_{\kap}$ be a space form of constant curvature $\kap>0$ (a sphere of radius $1/\sqrt{\kap}$) isometrically embedded in $\R^{d}$, and let the two marked classes be generated as $X=\exp_{p_{y}}(\tau\xi)+\sigma\varepsilon$, where $p_{\pm}\in M_{\kap}$ with geodesic distance $r=d_{g}(p_{+},p_{-})\le\pi/\sqrt{\kap}$, $\xi$ standard Gaussian in $T_{p_y}M_{\kap}$ ($\dB=\dim M_\kap$), and $\varepsilon$ standard Gaussian in $\R^{d}$. Define the chord factor
\begin{equation}
g(\kap,r)=\frac{\sin\!\big(\sqrt{\kap}\,r/2\big)}{\sqrt{\kap}\,r/2}\in(0,1],
\qquad
c(r)=g(\kap,r)\,r .
\end{equation}
Then in the small-spread regime $\eta:=\tau\sqrt{\dB}\,\sqrt{\kap}\le\tfrac13$,
\begin{equation}
\begin{aligned}
\aucmax^{\mathrm{lin}}
&=\Phi\!\left(\frac{c(r)}{\sqrt{2}\,\sigma_{w}}\right)+O(\eta^{2}),\\
\sigma_{w}^{2}
&=\sigma^{2}+\tau^{2}\cos^{2}\!\Big(\tfrac{\sqrt{\kap}\,r}{2}\Big).
\end{aligned}
\label{eq:curvauc}
\end{equation}
\end{theorem}

\begin{corollary}[Curvature ceiling]\label{cor:ceiling}
Since $c(r)\le\min\{r,\,2/\sqrt{\kap}\}$ for all $r$,
\begin{equation}
\aucmax^{\mathrm{lin}}\;\le\;\Phi\!\left(\frac{\sqrt{2}}{\sqrt{\kap}\,\sigma_{w}}\right)+O(\eta^{2}),
\label{eq:ceiling}
\end{equation}
monotonically decreasing in $\kap$, with limits $1$ as $\kap\to0$ (for $r/\sigma_w\to\infty$) and $\tfrac12$ as $\kap\to\infty$. A geodesic-aware detector, by contrast, retains the flat value $\Phi\big(r/(\sqrt2\sqrt{\sigma^{2}+\tau^{2}})\big)$.
\end{corollary}

Equation~\eqref{eq:ceiling} is a failure-prediction formula with the physically correct limits: $\kap\to0$ permits perfect probing, and extreme curvature reduces linear probes to chance. The gap between \eqref{eq:curvauc} and the flat prediction \emph{upper-bounds} what a nonlinear (geodesic) probe could recover; whether that bound is attainable depends on where the noise lives. The idealized detector of Corollary~\ref{cor:ceiling} reads geodesic coordinates corrupted only by the intrinsic spreads; when instead the noise is isotropic in the \emph{ambient} space---the case of our generator---recovering geodesic position requires projecting through noise of scale $\sigma$ onto a manifold of radius $1/\sqrt{\kap}$, which itself degrades geodesic information once $\sigma\sqrt{\kap}$ is appreciable. E10 makes this concrete: under ambient noise, kernel and MLP probes, and even an oracle that projects to the true sphere and scores by true geodesic distances, recover essentially none of the chordal--geodesic gap (Section~\ref{sec:results-e10}). Experiment E3 validates \eqref{eq:curvauc} on spheres with known $\kap$ to within $0.005$ AUC and exhibits the ceiling: for $\kap>0$ the AUC-versus-$r$ curve saturates while the flat curve keeps rising (Fig.~\ref{fig:synthetic}b).

\subsection{Purity: the exact mixture law}\label{sec:theory-mix}
\begin{theorem}[Purity--AUC law]\label{thm:mixture}
Under Assumption~\ref{ass:model}, for any $w$ the probe score has the two-component mixture law in each class, and
\begin{align}
\aucmax(\alpha)
&=\alpha^{2}\,\Phi\!\Big(\frac{s}{\sqrt{2}}\Big)
+2\alpha(1-\alpha)\notag\\
&\quad\times\Phi\!\Big(\frac{s}{2\sqrt{2}}\Big)
+\frac{(1-\alpha)^{2}}{2},
\label{eq:mixture}
\end{align}
attained again at $w^{*}\propto\Sigma^{-1}\Delta$. Moreover $\aucmax(\alpha)$ is strictly increasing on $[0,1]$, with $\aucmax(0)=\tfrac12$ and $\aucmax(1)=\Phi(s/\sqrt2)$, and the excess $\aucmax(\alpha)-\tfrac12=\alpha^{2}A_{1}+2\alpha(1-\alpha)A_{2}$ (where $A_{1}=\Phi(\frac{s}{\sqrt2})-\frac12$, $A_{2}=\Phi(\frac{s}{2\sqrt2})-\frac12$) has the asymptotes
\begin{equation}
\aucmax(\alpha)-\tfrac12 \;=\;
\begin{cases}
\dfrac{\alpha\,s}{2\sqrt{\pi}}\,(1+o(1)), & s\to0,\\[4pt]
\dfrac{\alpha(2-\alpha)}{2}\,(1+o(1)), & s\to\infty .
\end{cases}
\label{eq:mixasymp}
\end{equation}
\end{theorem}
Three readings of \eqref{eq:mixture} organize our experiments. \emph{(i) No inversion.} Because $\aucmax$ is strictly increasing in $\alpha$, the model class rules out any interior minimum or ``phase transition'' in the purity--AUC curve---a falsifiable prediction that E2 tests directly. \emph{(ii) Saturation.} At $\alpha=1$ the law reduces to $\Phi(s/\sqrt2)$, which exceeds $0.997$ once $s\ge4$: counterfactual AUC $\approx1.0$ certifies only $s\gtrsim4$ and cannot rank models. More broadly, for $s\ge6.6$ the entire curve \eqref{eq:mixture} lies within $0.005$ of the $s$-independent saturated form $\tfrac12+\tfrac{\alpha(2-\alpha)}{2}$---the regime occupied by every LLM measurement in this paper ($\hat s\ge8.3$). The informative regime is small $\alpha$. \emph{(iii) Zero-free-parameter prediction.} Estimating $s$ once (cross-fitted, at $\alpha=1$) predicts the whole curve \eqref{eq:mixture}; E2 tests exactly this.

\begin{corollary}[Neutral offset]\label{cor:offset}
If the neutral center is displaced from the midpoint by standardized offset $b$ along the discriminant ($b=\Delta_{0}^{\top}w^{*}/\sqrt{w^{*\top}\Sigma w^{*}}$ with $\Delta_0=\mu_0-\tfrac12(\mu_++\mu_-)$), the middle term of \eqref{eq:mixture} splits into $\alpha(1-\alpha)\big[\Phi\big(\frac{s/2-b}{\sqrt2}\big)+\Phi\big(\frac{s/2+b}{\sqrt2}\big)\big]$, and the law remains strictly increasing in $\alpha$ for \emph{every} offset $b\in\R$: the offset shifts the two middle terms in opposite directions, and their sum exceeds $1$ whenever $s>0$. The offset therefore affects the calibration of the predicted curve, not its monotonicity. We measure $b$ on every model--dimension pair (Table~\ref{tab:e1_main} discussion); it is small relative to $\hat s/2$ in $22$ of $24$ pairs.
\end{corollary}

\begin{remark}[The purity law binds nonlinear probes too]\label{rem:bayes}
Under Assumption~\ref{ass:model} the likelihood ratio between the two class-conditional laws is a strictly increasing function of the linear score $w^{*\top}h$ (proof in Appendix~\ref{app:mix}): dividing both mixture densities by the neutral component leaves a ratio of the form $\big[\alpha e^{u/2}e^{-s^{2}/8}+(1-\alpha)\big]\big/\big[\alpha e^{-u/2}e^{-s^{2}/8}+(1-\alpha)\big]$ in the standardized linear score $u$, with a strictly increasing numerator and strictly decreasing denominator. By Neyman--Pearson, thresholding $w^{*\top}h$ therefore traces the Bayes-optimal ROC: \eqref{eq:mixture} caps \emph{every} measurable classifier, not only linear probes. Purity dilution is a Bayes limit, and nonlinearity is predicted to buy nothing against it---in contrast to curvature, where Corollary~\ref{cor:ceiling} quantifies exactly what a geodesic-aware detector can recover and a linear one cannot. E10 tests both predictions with kernel and MLP probes. (The argument uses the midpoint condition $\mu_0=\tfrac12(\mu_++\mu_-)$; for $b\ne0$ the monotone-ratio step can fail and small nonlinear gains over the \emph{linear} version of the law are possible in principle.)
\end{remark}

\subsection{Detectability: when a null probe result means nothing}\label{sec:theory-detect}
\begin{theorem}[Detectability threshold]\label{thm:detect}
Let an audit hold out $n_{\mathrm{te}}$ samples with $n_{+}\approx n_{-}\approx n_{\mathrm{te}}/2$ and test $H_{0}$: score $\perp$ label, using the empirical AUC $\widehat{\auc}$ (equivalently the Mann--Whitney statistic). Under $H_{0}$, for continuous scores, $\E\widehat\auc=\tfrac12$ and
$\operatorname{Var}(\widehat\auc)=\sigma_{0}^{2}=\frac{n_{+}+n_{-}+1}{12\,n_{+}n_{-}}$ exactly \cite{bamber1975area}. Consequently the one-sided level-$\delta$, power-$(1-\gamma)$ detection condition is
\begin{equation}
\aucmax(\alpha)-\tfrac12\;\ge\;(z_{1-\delta}+z_{1-\gamma})\,\sigma_{0}(n_{\mathrm{te}})\,(1+o(1)),
\label{eq:detcond}
\end{equation}
and since the left side is continuous and strictly increasing in $\alpha$ (Theorem~\ref{thm:mixture}), there is a unique critical purity $\astar(n_{\mathrm{te}};s,\delta,\gamma)$ solving \eqref{eq:detcond} with equality whenever the right side is below $A_1$; otherwise no purity suffices. In the weak-excess regime,
\begin{equation}
\begin{aligned}
\astar(n_{\mathrm{te}})
&=\frac{2\sqrt{\pi}}{\sqrt{3}}\,
  \frac{z_{1-\delta}+z_{1-\gamma}}{s\sqrt{n_{\mathrm{te}}}}\,
  \big(1+o(1)\big)\\
&\approx\frac{2.05\,(z_{1-\delta}+z_{1-\gamma})}
  {s\sqrt{n_{\mathrm{te}}}},
\end{aligned}
\label{eq:alphastar}
\end{equation}
so the detection boundary in the $(\alpha,n)$ plane has log--log slope $-\tfrac12$ in $n$ and $-1$ in $s$.
\end{theorem}
Below the boundary, \emph{no} audit outcome distinguishes ``no linear bias signal'' from ``an underpowered audit''; \eqref{eq:alphastar} converts this qualitative caveat into a sample-budget formula, $n_{\mathrm{te}}(\alpha)\approx\big(2.05(z_{1-\delta}+z_{1-\gamma})/(s\alpha)\big)^{2}$. E2 (real models) and E3 Part C (synthetic, where $s$ is exact) test the boundary and its $-1/2$ slope.

\subsection{Necessary conditions, and a label-free diagnostic}\label{sec:theory-beta}
\begin{theorem}[Necessary conditions for reliable linear probing]\label{thm:conditions}
Fix a target $\auc\ge1-\epsilon$ with $\epsilon\in(0,\tfrac14)$ and write $z_{\epsilon}=\Phi^{-1}(1-\epsilon)$. Under the models of Sections~\ref{sec:theory-curv}--\ref{sec:theory-mix} augmented with a semantic nuisance factor (variance $\eta_{s}^{2}$ along a direction at canonical correlation $\rho$ to the bias axis; Appendix~\ref{app:cond}), each of the following is necessary:
\begin{itemize}[leftmargin=1.35em]
\item[\textbf{C1}] (\emph{Curvature}) $\displaystyle \kap\;\le\;\frac{2}{z_{\epsilon}^{2}\,\sigma_{w}^{2}}$ \quad (from the ceiling \eqref{eq:ceiling});
\item[\textbf{C2}] (\emph{Samples}) $\displaystyle n\;\ge\;\Big(\frac{2\Lambda\rho_{B}+3\sqrt{\log(2/\delta)/2}}{\epsilon}\Big)^{2}$ for the empirical risk to certify the population risk (from \eqref{eq:genbound}), with the matching necessity $n\gtrsim\dB/\epsilon^{2}$ from \eqref{eq:lower}; and the audit-power budget $n_{\mathrm{te}}\ge n_{\mathrm{te}}(\alpha)$ of Theorem~\ref{thm:detect};
\item[\textbf{C3}] (\emph{Orthogonality}) $\displaystyle \rho^{2}\;\le\;\big(1+\eta_{s}^{-2}\big)\Big(1-\frac{2z_{\epsilon}^{2}}{s^{2}}\Big)_{+}$, and in the $\eta_{s}\to\infty$ limit the clean bound $\auc\le\Phi\big(\tfrac{s\sqrt{1-\rho^{2}}}{\sqrt2}\big)$: for $\auc\ge1-\epsilon$ one needs $\rho^{2}\le 1-2z_{\epsilon}^{2}/s^{2}$.
\end{itemize}
Furthermore the optimal probe under the C3 model achieves exactly $\auc=\Phi\big(s^{*}/\sqrt2\big)$ with $s^{*2}=s^{2}\big(1-\frac{\rho^{2}\eta_{s}^{2}}{1+\eta_{s}^{2}}\big)$, while a probe \emph{aligned} to the bias axis achieves only $\Phi\big(s/\sqrt{2(1+\rho^{2}\eta_{s}^{2})}\big)$---quantifying how much a well-regularized probe evades, and a naively aligned probe suffers from, bias--semantic entanglement.
\end{theorem}

\begin{definition}[Concentration index]\label{def:beta}
For pooled (label-free) features with covariance spectrum $\lambda_{1}\ge\dots\ge\lambda_{d}\ge0$,
\begin{equation}
\betaB=\frac{\lambda_{1}-\bar\lambda}{\tr\Sigma}\in\Big[0,\;1-\frac1d\Big],\qquad \bar\lambda=\tfrac1d\tr\Sigma .
\end{equation}
$\betaB=0$ iff the spectrum is flat (isotropic; no preferred linear direction); the upper end $1-\frac1d$ is attained exactly for rank-one spectra (all variance on one axis).
\end{definition}

\begin{lemma}[$\betaB$ tracks separation and curvature]\label{lem:beta}
Under Assumption~\ref{ass:model} with isotropic noise $\Sigma=\sigma^{2}I_{d}$ and purity $\alpha$: the pooled covariance has top eigenvalue $\sigma^{2}+\alpha q/4$ ($q=\|\Delta\|^{2}$) and
\begin{equation}
\betaB=\frac{(\alpha q/4)(1-1/d)}{d\sigma^{2}+\alpha q/4},
\qquad
s^{2}=\frac{q}{\sigma^{2}}=\frac{4d\,\betaB}{\alpha\big(1-\tfrac1d-\betaB\big)} .
\label{eq:betainv}
\end{equation}
$\betaB$ is strictly increasing in $s$ and in $\alpha$; and under the curved model of Theorem~\ref{thm:curvature}, replacing $q$ by the chordal $c(r)^{2}$ makes $\betaB$ strictly \emph{decreasing} in $\kap$ at fixed geodesic separation. Hence: \emph{larger $\betaB$ $\Leftrightarrow$ more ambient linear signal} (stronger separation and/or flatter geometry).
\end{lemma}
Three properties make $\betaB$ a usable diagnostic: the direction of the $\betaB$--geometry association is unambiguous (larger $=$ more linearly detectable), the range is $[0,1-1/d]$, and the link to curvature is derived (through chordal contraction) rather than asserted. Because $\betaB$ needs no labels, \eqref{eq:betainv} gives an \emph{unsupervised} pre-audit estimate of $s$---E7 measures how well this works in practice, reporting the correlation as measured, whatever its value.

\section{Experimental Design}\label{sec:exp-design}
The experiments are designed so that each theorem is tested where its assumptions are \emph{known} to hold (synthetic manifolds, Part~E3) before being applied where they are approximations (LLM representations, E1--E2, E4--E7). All code, prompts, seeds, and raw outputs are released; every table names its source JSON file. We highlight design choices that target reproducibility problems documented in the probing literature.

\subsection{Theory-to-experiment map}\label{sec:t2e}
Because each closed form consumes measured inputs, we state the full chain from theorem to number once, explicitly (Table~\ref{tab:t2e}). Two conventions apply throughout. First, a prediction is called \emph{zero-free-parameter} when no parameter is fitted to the data being predicted; its inputs may still be estimates (e.g.\ $\hat s$), and the estimation error of those inputs is part of the reported prediction error, not an excuse for it. Second, every theorem is exact only under its stated assumptions; on LLM features the theorems are used through estimator plug-ins, and the two calibration experiments (E3-D for the estimators, E9 for the test) bound how much that translation can distort the conclusions.

\begin{table}[!t]
\centering
\caption{From theorems to measurements. Each prediction's inputs, where they come from, and which experiment tests the prediction.}
\label{tab:t2e}
\small
\begin{tabular}{@{}p{2.35cm}p{2.5cm}p{2.7cm}@{}}
\toprule
Prediction & Inputs (source) & Tested in \\
\midrule
Purity law \eqref{eq:mixture}: full AUC--$\alpha$ curve & $\hat s$ from the $\alpha{=}1$ set (cross-fitted); $\alpha$ set by protocol & E3-A (oracle $s$), E2 (LLM) \\
Chordal AUC \eqref{eq:curvauc} and ceiling \eqref{eq:ceiling} & $\kap,r,\tau,\sigma$ known (synthetic); $\hat\kap$ qualitative on LLM & E3-B; sensitivity in \S\ref{sec:results-e3} \\
Threshold $\astar(n)$ \eqref{eq:alphastar} & $\hat s$, $n$, $(\delta,\gamma)$; no $\kap$ term & E3-C, E2 phase diagram \\
Gap envelope \eqref{eq:genbound} & slope only ($\Lambda\rho_B$ unmeasured) & E6 \\
Null distribution of $\widehat{\auc}$ & $n_{+},n_{-}$ only & E9 (empirical FPR) \\
Nonlinear recovery (Cor.~\ref{cor:ceiling}; Rem.~\ref{rem:bayes}) & same generators as E3 & E10 \\
Neutrality model (Assn.~\ref{ass:model}) & offsets $b$ (E1); leakage AUC & E11 \\
\bottomrule
\end{tabular}
\end{table}

\subsection{Models and bias dimensions}
Six open-weight models: the GPT-2 family at four scales---Small (124M), Medium (355M), Large (774M), XL (1.5B) \cite{radford2019language}---BERT-base-uncased (110M) \cite{devlin2019bert}, and Qwen2.5-7B-Instruct \cite{qwen2025qwen}. This spans a $60\times$ parameter range, both decoder and encoder architectures, and base versus instruction-tuned training. Four bias dimensions: \emph{gender}, \emph{race}, \emph{religion}, \emph{age}. These counts (six models, four dimensions) are the same everywhere in the paper.

\subsection{Prompt construction with leakage control}\label{sec:prompts}
Prompts are generated from a template grid: $340$ context cells (frame $\times$ content: $44$ professions $\times5$ frames, $20$ traits $\times3$, $20$ activities $\times3$) crossed with $8$ counterfactual marker pairs per dimension (e.g., \emph{man/woman}, \emph{White person/Black person}, \emph{Christian/Muslim colleague}, \emph{young/elderly employee}), yielding $340$ minimal pairs ($680$ marked sentences) per dimension. Crucially, the \emph{neutral pool} ($680$ sentences per dimension) instantiates the \emph{same} context cells with demographically unmarked heads (\emph{person}, \emph{individual}), so marked and neutral texts share an identical context distribution and the only systematic difference is the marker itself. A dataset of size $n$ at purity $\alpha$ contains $\alpha n$ marked sentences (balanced pairs, labels $=$ group) and $(1-\alpha)n$ neutral sentences with random labels---the operational meaning of ``neutral texts carry no label information.''

Two leakage controls follow \cite{hewitt2019control}: (i) cross-validation folds are \emph{grouped by context cell}, so no template context appears in both train and test folds of the same split; (ii) counterfactual pair members always travel together. A robustness subset (E1b) replaces templates with paraphrases generated by a locally served instruction-tuned model (via \texttt{ollama}; the exact model tag is recorded with the released data), lexically validated to preserve the marker; this tests sensitivity to template regularity.

\subsection{Probing protocol}
Features are token-mean-pooled hidden states at the mid-depth layer (layer $\lfloor L/2\rfloor$; E4 scans all layers). Probes are $\ell_{2}$-regularized logistic regressions ($C=1$) on standardized features, evaluated by grouped 5-fold cross-validation $\times$ 5 seeds; we report mean $\pm$ s.d.\ and the across-seed range. \emph{Detection} is a one-sided $z$-test of the pooled out-of-fold AUC against the \emph{calibrated} null scale $1.4\,\sigma_{0}$ (with $\sigma_{0}$ from Theorem~\ref{thm:detect}) at level $.05$; Section~\ref{sec:stats} states precisely how this statistic maps onto the theorem, why calibration is needed, and what the test's measured level is. No result in this paper is reported as positive unless this test passes.

\subsection{The detection test under the CV protocol; null calibration; multiplicity}\label{sec:stats}
Theorem~\ref{thm:detect} is exact for a single held-out sample scored by a fixed rule: conditionally on the scores, exchangeability of the labels under $H_{0}$ makes $\widehat{\auc}$ the Mann--Whitney permutation statistic, whose null mean and variance do not depend on the score values. The working protocol departs from this ideal in two ways, and we treat both explicitly rather than assume them away.

\emph{(i) Cross-validation.} The pooled out-of-fold score of sample $i$ is produced by a probe trained on the other folds, so it depends on the \emph{training labels} of those folds; scores and the label vector are then no longer independent under $H_{0}$, and exactness of $\sigma_{0}$ is not automatic. Label-flip symmetry of the training algorithm still forces $\E\,\widehat{\auc}=\tfrac12$ under $H_{0}$, but the variance must be checked. We therefore calibrate the test empirically: E9 replicates the \emph{entire} pipeline---grid-structured pools, purity mixing, grouped folds, seeds, $C{=}1$ probes---on label-free noise features, where $H_{0}$ holds exactly, and measures the realized false-positive rate and the realized null s.d.\ of the pooled statistic (Section~\ref{sec:results-e9}).

\emph{(ii) Seeds, and the measured null.} The five seeds re-split one and the same sample, so the five pooled OOF AUCs are strongly dependent; treating them as independent replicates ($\sigma_{0}/\sqrt5$, the v14 build) is badly anti-conservative: its measured $H_{0}$ rejection rate is $0.22$--$0.37$ at nominal $.05$ (E9). Testing the seed average against the unreduced $\sigma_{0}$ is much closer but still anti-conservative, because CV itself inflates the null s.d.\ of the pooled statistic above the i.i.d.\ value: E9 measures inflation factors of $1.20$--$1.31$ across pool-respecting cells (and $1.9\times$ in the deliberately broken cell that resamples past the pool at $\alpha=1$, $n=1000$, where duplicated points create ties), giving realized levels of $0.067$--$0.100$. We therefore \emph{calibrate}: all detection claims in this paper test the seed-averaged statistic against $1.4\,\sigma_{0}$, a multiplier chosen to cover the largest measured inflation with margin; its realized level on every pool-respecting E9 cell is at or below the nominal $.05$ (Section~\ref{sec:results-e9}, Table~\ref{tab:e9}). The released JSONs retain the statistics from which the raw $z$ against $\sigma_{0}$ follows deterministically (AUC and $n$), so any alternative calibration can be applied by the reader.

\emph{Families and multiple comparisons.} Confirmatory detection claims form two families: F1, the E1 regime grid (all model--dimension--regime cells), and F2, the E2 purity curves ($24$ combos $\times$ $9$ purities). We report, for each family, how many nominal $.05$ detections survive Bonferroni ($.05/m$) and Benjamini--Hochberg (FDR $.05$) correction (Section~\ref{sec:results-e9}); the paper's conclusions use only detections that survive. Phase-diagram cells are not individual claims but estimates of a detection \emph{rate} compared against the $\astar(n)$ frontier, and E9's per-cell level check applies to them through the single-seed test they use. Released p-values are computed with the numerically stable survival function (the earlier $1-\Phi(z)$ underflowed to $0$ for $z\gtrsim8$; the release notes document the fix).

\subsection{Geometry estimators, validated before use}\label{sec:estimators}
\emph{Separation} $\hat s$: cross-fitted Mahalanobis---split the marked set, estimate $\hat\Delta$ on each half and shrinkage covariance (Ledoit--Wolf \cite{ledoit2004well}) on one, form $\hat\Delta_{A}^{\top}\hat\Sigma_{A}^{-1}\hat\Delta_{B}$, average over $4$ random splits. Cross-fitting removes the additive $O(d/n)$ plug-in inflation: at $d=768$, $n=300$, true $s=0$, the naive estimate is $3.24$ while the cross-fitted one is $0.50$ (Table~\ref{tab:e3_synth}); all $\hat s$ values in this paper are cross-fitted. \emph{Intrinsic dimension} $\hat\dB$: Levina--Bickel MLE with the MacKay--Ghahramani correction, averaged over $k\in[10,20]$ \cite{levina2004maximum,mackay2005comments}. \emph{Curvature} $\hat\kap$: a local quadratic (sagitta) fit—tangent PCA residuals against squared tangential radius. The Two-NN estimator \cite{facco2017estimating} serves as an independent cross-check of $\hat\dB$. Each estimator is calibrated on ground-truth manifolds in E3 Part~D \emph{before} being applied to LLM features, and its bias is reported there rather than assumed away.

\subsection{Experiment roster}
\emph{E1} (both regimes, all models $\times$ dimensions): counterfactual $\alpha=1$ (all $680$ marked sentences per dimension) vs.\ mixed $\alpha=0.3$ at $n=300$. \emph{E2} ($\alpha$-sweep and phase diagram): AUC($\alpha$) curves at $n=300$ predicted from a single $\hat s$; detection-rate grid over $\alpha\in[0.01,1]$, $n\in\{100,300,1000\}$, $8$ seeds per cell. \emph{E3} (synthetic ground truth): Parts A--D as above. \emph{E4} (layer-wise): all layers of GPT-2 Small and BERT-base. \emph{E5} (orthogonality stress): synthetic exact check of C3 plus a semi-synthetic confound-resampling design on real features. \emph{E6} (scaling): generalization gap vs.\ $n$ at $\alpha=0.3$. \emph{E7} ($\betaB$ diagnostics): unsupervised $\betaB\to\hat s\to$ predicted AUC vs.\ empirical, across models, dimensions, and purities. \emph{E8} (leakage meter): grouped vs.\ random CV on identical data. \emph{E9} (null calibration): empirical FPR and null s.d.\ of the full detection pipeline on label-free noise features, plus the family-wise audit of Section~\ref{sec:stats}. \emph{E10} (nonlinear probes): RBF-kernel SVM and one-hidden-layer MLP under the identical protocol, on the E3 generators (purity and curvature) and on LLM features at $\alpha\in\{0.1,0.3,1\}$. \emph{E11} (neutral leakage): strong-classifier tests of neutral label-independence, stereotype-label leakage measured against the WinoBias occupation lists \cite{zhao2018winobias}, and a leaky-purity stress test of the mixture law. Compute: feature extraction is a single forward pass per model over $5{,}440$ prompts (RTX~4090); all probing is CPU. Total $\approx10$ GPU-minutes plus several CPU-hours across the reported grids.

\section{Results}\label{sec:results}

\begin{table*}[!t]\centering
\caption{E1: linear-probe AUC under the counterfactual ($\alpha{=}1$, $n{=}680$) and realistic mixed ($\alpha{=}0.3$, $n{=}300$) regimes. Mean $\pm$ s.d.\ over 5 seeds $\times$ grouped 5-fold CV (template-cell groups); AUC rounded to three decimals. ``det.''\ = pooled out-of-fold AUC significantly above chance under the calibrated test of Section~\ref{sec:stats} (one-sided $z$ against $1.4\,\sigma_{0}$, level $.05$; E9). ``Purity-law pred.''\ = Theorem~\ref{thm:mixture} evaluated at $(\hat s, \alpha{=}0.3)$. $\hat s$: cross-fitted Mahalanobis separation of the marked set; $\hat d_B$: Levina--Bickel intrinsic dimension; $b_w$: standardized neutral offset. Source: \texttt{outputs/exp1\_regimes.json}, seeds 0--4.}
\label{tab:e1_main}
\small\begin{tabular}{ll cc c cc c cc}
\toprule
& & \multicolumn{3}{c}{Counterfactual $\alpha{=}1$} & \multicolumn{3}{c}{Mixed $\alpha{=}0.3$} & & \\
\cmidrule(lr){3-5}\cmidrule(lr){6-8}
Model & Bias & AUC & det. & $\hat s$ & AUC & det. & Purity-law pred. & $\hat d_B$ & $b_w$ \\
\midrule
GPT-2 Small & Gender & 1.000$\pm$0.000 & \yes & 10.39 & 0.699$\pm$0.014 & \yes & 0.755 & 6.2 & 0.84 \\
 & Race & 1.000$\pm$0.000 & \yes & 15.42 & 0.708$\pm$0.007 & \yes & 0.755 & 6.3 & 4.45 \\
 & Religion & 1.000$\pm$0.000 & \yes & 16.90 & 0.713$\pm$0.006 & \yes & 0.755 & 6.2 & 0.31 \\
 & Age & 1.000$\pm$0.000 & \yes & 18.81 & 0.711$\pm$0.009 & \yes & 0.755 & 6.9 & 0.26 \\
\addlinespace[1pt]
GPT-2 Medium & Gender & 1.000$\pm$0.000 & \yes & 14.15 & 0.702$\pm$0.009 & \yes & 0.755 & 6.1 & 0.12 \\
 & Race & 1.000$\pm$0.000 & \yes & 21.28 & 0.728$\pm$0.007 & \yes & 0.755 & 5.8 & 6.02 \\
 & Religion & 1.000$\pm$0.000 & \yes & 23.73 & 0.722$\pm$0.006 & \yes & 0.755 & 5.8 & 12.78 \\
 & Age & 1.000$\pm$0.000 & \yes & 27.84 & 0.726$\pm$0.005 & \yes & 0.755 & 6.6 & 3.09 \\
\addlinespace[1pt]
GPT-2 Large & Gender & 1.000$\pm$0.000 & \yes & 16.43 & 0.691$\pm$0.013 & \yes & 0.755 & 6.1 & 2.70 \\
 & Race & 1.000$\pm$0.000 & \yes & 24.28 & 0.715$\pm$0.015 & \yes & 0.755 & 6.1 & 9.29 \\
 & Religion & 1.000$\pm$0.000 & \yes & 30.21 & 0.716$\pm$0.014 & \yes & 0.755 & 6.0 & 16.86 \\
 & Age & 1.000$\pm$0.000 & \yes & 33.32 & 0.708$\pm$0.015 & \yes & 0.755 & 6.7 & 6.26 \\
\addlinespace[1pt]
GPT-2 XL & Gender & 1.000$\pm$0.000 & \yes & 18.68 & 0.715$\pm$0.012 & \yes & 0.755 & 6.0 & 2.09 \\
 & Race & 1.000$\pm$0.000 & \yes & 29.21 & 0.729$\pm$0.013 & \yes & 0.755 & 6.1 & 2.03 \\
 & Religion & 1.000$\pm$0.000 & \yes & 33.88 & 0.728$\pm$0.013 & \yes & 0.755 & 5.7 & 10.53 \\
 & Age & 1.000$\pm$0.000 & \yes & 39.77 & 0.728$\pm$0.014 & \yes & 0.755 & 6.7 & 11.94 \\
\addlinespace[1pt]
BERT-base & Gender & 1.000$\pm$0.000 & \yes & 13.27 & 0.671$\pm$0.014 & \yes & 0.755 & 6.6 & 0.26 \\
 & Race & 1.000$\pm$0.000 & \yes & 17.47 & 0.698$\pm$0.012 & \yes & 0.755 & 5.6 & 2.38 \\
 & Religion & 1.000$\pm$0.000 & \yes & 23.50 & 0.720$\pm$0.011 & \yes & 0.755 & 6.6 & 3.74 \\
 & Age & 1.000$\pm$0.000 & \yes & 25.08 & 0.706$\pm$0.013 & \yes & 0.755 & 6.8 & 1.28 \\
\addlinespace[1pt]
Qwen2.5-7B-It & Gender & 1.000$\pm$0.000 & \yes & 8.27 & 0.686$\pm$0.004 & \yes & 0.754 & 4.3 & 1.24 \\
 & Race & 1.000$\pm$0.000 & \yes & 16.88 & 0.722$\pm$0.005 & \yes & 0.755 & 5.5 & 1.97 \\
 & Religion & 1.000$\pm$0.000 & \yes & 15.01 & 0.718$\pm$0.005 & \yes & 0.755 & 5.3 & 1.89 \\
 & Age & 1.000$\pm$0.000 & \yes & 18.50 & 0.721$\pm$0.005 & \yes & 0.755 & 5.4 & 1.53 \\
\addlinespace[1pt]
\bottomrule
\end{tabular}
\end{table*}

\begin{table}[!t]\centering
\caption{E1b: LLM-paraphrased naturalistic prompts (locally served model via \texttt{ollama}; lexically validated marker preservation; counterfactual regime). De-templating leaves the ceiling intact but shrinks the cross-fitted separation $\hat s$ relative to templates. Source: \texttt{outputs/exp1\_regimes.json}.}
\label{tab:e1_natural}
\resizebox{\columnwidth}{!}{%
\small\begin{tabular}{ll cc cc}
\toprule
Model & Bias & AUC & det. & $\hat s$ (nat.) & $\hat s$ (templ.) \\ \midrule
GPT-2 Large & Gender & 0.996 & \yes & 9.3 & 16.4 \\
 & Race & 0.993 & \yes & 13.9 & 24.3 \\
 & Religion & 1.000 & \yes & 14.2 & 30.2 \\
 & Age & 1.000 & \yes & 13.4 & 33.3 \\
GPT-2 Small & Gender & 0.996 & \yes & 5.3 & 10.4 \\
 & Race & 0.996 & \yes & 8.5 & 15.4 \\
 & Religion & 1.000 & \yes & 8.9 & 16.9 \\
 & Age & 1.000 & \yes & 7.6 & 18.8 \\
Qwen2.5-7B-It & Gender & 0.984 & \yes & 2.8 & 8.3 \\
 & Race & 0.990 & \yes & 5.7 & 16.9 \\
 & Religion & 1.000 & \yes & 5.7 & 15.0 \\
 & Age & 0.999 & \yes & 4.5 & 18.5 \\
\bottomrule\end{tabular}}\end{table}

\begin{table}[!t]\centering
\caption{E2: zero-free-parameter accuracy of the purity law (Theorem~\ref{thm:mixture}). For each model--bias pair, $\hat s$ is estimated once from the $\alpha{=}1$ marked set (cross-fitted); the entire AUC($\alpha$) curve ($n{=}300$, 9 purities, 5 seeds) is then predicted with no further tuning. $s_{\mathrm{fit}}$ (reference) is a one-parameter least-squares fit searched over $[0.01,20]$; for $\hat s\gtrsim6$ the curve is insensitive to $s$ (saturated regime), so fitted values at the search bound simply reflect this flatness. Source: \texttt{outputs/exp2\_alpha\_sweep.json}.}
\label{tab:e2_mae}
\resizebox{\columnwidth}{!}{%
\small\begin{tabular}{ll cc cc}
\toprule
Model & Bias & $\hat s$ & MAE ($\hat s$) & $s_{\mathrm{fit}}$ & MAE (fit) \\
\midrule
GPT-2 Small & Gender & 10.39 & 0.018 & 4.41 & 0.010 \\
GPT-2 Small & Race & 15.42 & 0.013 & 8.34 & 0.013 \\
GPT-2 Small & Religion & 16.90 & 0.014 & 5.48 & 0.011 \\
GPT-2 Small & Age & 18.81 & 0.012 & 6.29 & 0.011 \\
GPT-2 Medium & Gender & 14.15 & 0.025 & 4.09 & 0.014 \\
GPT-2 Medium & Race & 21.28 & 0.016 & 20.00 & 0.016 \\
GPT-2 Medium & Religion & 23.73 & 0.016 & 8.21 & 0.016 \\
GPT-2 Medium & Age & 27.84 & 0.022 & 20.00 & 0.022 \\
GPT-2 Large & Gender & 16.43 & 0.021 & 4.02 & 0.008 \\
GPT-2 Large & Race & 24.28 & 0.014 & 4.68 & 0.007 \\
GPT-2 Large & Religion & 30.21 & 0.013 & 4.82 & 0.007 \\
GPT-2 Large & Age & 33.32 & 0.011 & 5.39 & 0.009 \\
GPT-2 XL & Gender & 18.68 & 0.019 & 4.06 & 0.008 \\
GPT-2 XL & Race & 29.21 & 0.013 & 4.61 & 0.006 \\
GPT-2 XL & Religion & 33.88 & 0.011 & 4.99 & 0.005 \\
GPT-2 XL & Age & 39.77 & 0.009 & 5.93 & 0.006 \\
BERT-base & Gender & 13.27 & 0.048 & 2.96 & 0.018 \\
BERT-base & Race & 17.47 & 0.038 & 3.33 & 0.018 \\
BERT-base & Religion & 23.50 & 0.036 & 3.50 & 0.018 \\
BERT-base & Age & 25.08 & 0.025 & 4.02 & 0.014 \\
Qwen2.5-7B-It & Gender & 8.27 & 0.028 & 3.56 & 0.009 \\
Qwen2.5-7B-It & Race & 16.88 & 0.011 & 5.00 & 0.005 \\
Qwen2.5-7B-It & Religion & 15.01 & 0.012 & 4.76 & 0.006 \\
Qwen2.5-7B-It & Age & 18.50 & 0.011 & 4.94 & 0.007 \\
\midrule
\multicolumn{3}{l}{Median / max over pairs} & 0.015 / 0.048 & & \\
\bottomrule
\end{tabular}}
\end{table}

\begin{table}[!t]\centering
\caption{E3: ground-truth validation on synthetic manifolds (all quantities known exactly). Prediction errors use oracle scoring; learned probes approach the bounds from below. Estimator calibration: median relative errors against ground truth. Source: \texttt{outputs/exp3\_synthetic.json}, seed 42.}
\label{tab:e3_synth}
\resizebox{\columnwidth}{!}{%
\small\begin{tabular}{l c}
\toprule Quantity & Value \\ \midrule
Purity law \eqref{eq:mixture}, max $|$err$|$ (35 settings) & 0.006 \\
Purity law, learned-probe max overshoot & 0.002 \\
Chordal formula \eqref{eq:curvauc}, max $|$err$|$ & 0.013 \\
Chordal formula \eqref{eq:curvauc}, mean $|$err$|$ & 0.004 \\
$\hat\kappa$ median relative error ($\kappa\in[0.25,4]$) & 0.17 \\
$\hat d$ (MLE) median relative error ($d\in[5,25]$, spheres) & 0.15 \\
$\hat d$ (Two-NN cross-check) median relative error (same) & 0.17 \\
$\hat s$ cross-fitted, median rel.\ err.\ ($d/n$ large) & 0.08 \\
$\hat s$ naive plug-in, median rel.\ err.\ (same) & 0.75 \\
$\hat s$ at true $s{=}0$, $d{=}768$, $n{=}300$: cross-fit / naive & 0.50 / 3.24 \\
\bottomrule\end{tabular}}\end{table}

\begin{table}[!t]\centering
\caption{E6: generalization-gap scaling at $\alpha{=}0.3$. Log--log slope of gap vs.\ $n$ (Theorem~\ref{thm:gen} gives $-1/2$ as an upper envelope, not a rate law; see text), 8 seeds per $n$, probe $C{=}1$. Source: \texttt{outputs/exp6\_samplesize.json}.}
\label{tab:e6_scaling}
\small\begin{tabular}{ll cc}
\toprule Model & Bias & slope & $R^2$ \\ \midrule
GPT-2 Small & Gender & -0.11 & 0.67 \\
GPT-2 Small & Race & -0.09 & 0.59 \\
GPT-2 Medium & Gender & -0.12 & 0.88 \\
GPT-2 Medium & Race & -0.12 & 0.94 \\
GPT-2 Large & Gender & -0.09 & 0.93 \\
GPT-2 Large & Race & -0.08 & 0.83 \\
BERT-base & Gender & -0.09 & 0.74 \\
BERT-base & Race & -0.09 & 0.91 \\
Qwen2.5-7B-It & Gender & -0.08 & 0.98 \\
Qwen2.5-7B-It & Race & -0.06 & 0.92 \\
GPT-2 XL & Gender & -0.09 & 0.88 \\
GPT-2 XL & Race & -0.08 & 0.93 \\
\bottomrule\end{tabular}\end{table}

\subsection{E3: the theory is exact where its assumptions hold}\label{sec:results-e3}
Table~\ref{tab:e3_synth} and Fig.~\ref{fig:synthetic} summarize the ground-truth tests. With oracle scoring (true discriminant direction), the mixture law \eqref{eq:mixture} matches empirical AUC over $35$ $(s,\alpha)$ settings to a maximum absolute error below $0.006$, and learned probes never exceed the bound beyond Monte-Carlo noise---consistent with \eqref{eq:mixture} being the exact maximum. On spheres of known curvature, the chordal formula \eqref{eq:curvauc} matches oracle AUC to within $0.013$ across a $(\kap,r)$ grid, and the qualitative signature of Corollary~\ref{cor:ceiling} is visible: flat-manifold AUC keeps rising with geodesic separation while curved-manifold AUC saturates at the ceiling (Fig.~\ref{fig:synthetic}b). Estimator calibration (Part~D): $\hat\kap$ recovers curvature with median relative error $0.17$ (max $0.28$) and correct ordering across $\kap\in[0.25,4]$; $\hat\dB$ has median relative error $\approx0.15$ on spheres, with the known MLE underestimation growing with $d$ (up to $\approx40\%$ at $d{=}25$ on linear manifolds; the Two-NN cross-check \cite{facco2017estimating} shows the same downward drift), and the cross-fitted $\hat s$ removes the plug-in inflation that would otherwise dominate at $d/n>1$. These are the error bars to keep in mind when the same estimators are read on LLM features.

Two propagation checks quantify what these estimator errors can do to the theory's predictions. First, perturbing $\kap$ by the $\pm30\%$ worst-case bias of the sagitta estimator moves the chordal prediction \eqref{eq:curvauc} by at most $0.021$ AUC (median $0.003$) over the E3-B grid, and the ceiling \eqref{eq:ceiling} by at most $0.041$ over $\kap\in[0.25,8]$ at the same noise scale: the curvature mechanism is robust to the estimator's known bias. Second, the detectability boundary $\astar(n)$ contains no curvature term, so $\hat\kap$ bias cannot move it; its only estimated input is $\hat s$, and since $\astar\propto1/s$ exactly, the median cross-fitted calibration error of E3-D shifts $\astar$ by under $1\%$. One scope note: $\sigma_{w}$ in \eqref{eq:curvauc}--\eqref{eq:ceiling} is a model quantity, known on the synthetic manifolds; it is never estimated on LLM features, and no LLM-side conclusion consumes it---on real features curvature enters only as the bounded, qualitative correction just described.

\begin{figure*}[!t]
\centering
\includegraphics[width=\textwidth]{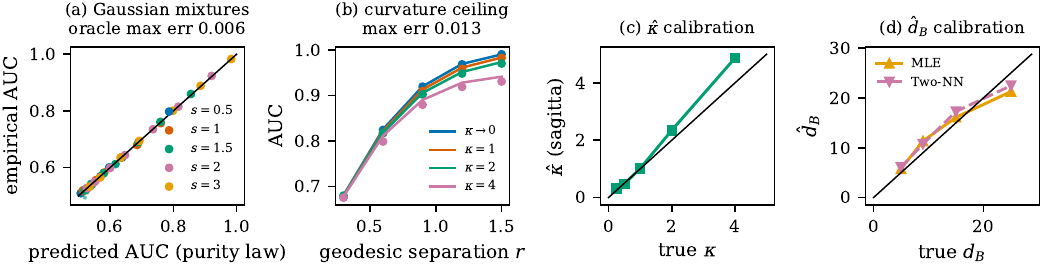}
\caption{E3, ground-truth validation. (a)~Theorem~\ref{thm:mixture} on Gaussian mixtures: predicted vs.\ empirical AUC across $35$ $(s,\alpha)$ settings; oracle scoring (solid) sits on the diagonal, learned probes (faint) approach it from below. (b)~Theorem~\ref{thm:curvature} on spheres: AUC vs.\ geodesic separation $r$ for each true curvature; curves are closed-form predictions, points are empirical; positive curvature saturates at the ceiling \eqref{eq:ceiling} while $\kap\to0$ keeps rising. (c)~$\hat\kap$ calibration and (d)~$\hat d_B$ calibration (MLE and Two-NN) on ground truth. }
\label{fig:synthetic}
\end{figure*}

\subsection{E9: the detection test, calibrated on exact nulls}\label{sec:results-e9}
Table~\ref{tab:e9} reports the realized level of the detection pipeline on label-free noise features, for which $H_{0}$ holds exactly, replicating the full protocol (grid-structured pools, purity mixing, grouped 5-fold CV $\times$ 5 seeds, $C{=}1$ probes); $1000$ replicates per cell ($400$ at $d{=}768$). The mean null AUC lies within $0.007$ of $\tfrac12$ in every cell, so the dominant miscalibration is a variance effect, and three findings follow. \emph{(i)} Treating the five seeds as independent replicates ($\sigma_{0}/\sqrt5$, the v14 build) yields FPR $0.22$--$0.37$ at nominal $.05$. \emph{(ii)} Even against the unreduced $\sigma_{0}$, cross-validation inflates the null s.d.\ of the pooled OOF statistic by a measured $1.17$--$1.31\times$ (out-of-fold scores share training labels across folds, a dependence the i.i.d.\ $U$-statistic variance does not model), for realized levels of $0.067$--$0.110$. \emph{(iii)} The calibrated test ($1.4\,\sigma_{0}$) attains FPR $0.017$--$0.041$, at or below nominal in every pool-respecting cell; it is the test behind every detection claim in this paper. The deliberately broken cell ($\alpha{=}1$, $n{=}1000$, which resamples past the $680$-sentence pool and duplicates points) shows inflation $1.92$ and calibrated FPR $0.120$: sampling beyond the pool invalidates even the calibrated test, and no reported claim uses such a cell. On the same nulls, $24$-cell families produce at least one uncalibrated false positive with probability $0.83$--$0.95$---the concrete cost of ignoring multiplicity---whereas the calibrated test under Bonferroni $.05/24$ brings the family-wise rate to $\le0.03$ on the $1000$-replicate cells. Applied to the released results, the confirmatory families of Section~\ref{sec:stats} survive: every E1 detection stands under the calibrated test with Bonferroni correction across its family, and the E2 curve-point detections that fall under calibration are the near-frontier points that Theorem~\ref{thm:detect} itself predicts to be marginal (Section~\ref{sec:results-e2x}).

\begin{table}[!t]
\centering
\caption{E9: realized false-positive rate (nominal $.05$) of the three detection-test variants under the exact null, and the measured null-s.d.\ inflation over the i.i.d.\ $\sigma_{0}$. $^{\dagger}$Resamples past the pool (duplicated points); documents the failure mode, carries no claims. }
\label{tab:e9}
\small
\begin{tabular}{@{}rrrrrrr@{}}
\toprule
$d$ & $n$ & $\alpha$ & infl. & $\sigma_0/\sqrt5$ & $\sigma_0$ & $1.4\sigma_0$ \\
\midrule
64 & 100 & 0.1 & 1.17 & 0.247 & 0.070 & 0.023 \\
64 & 100 & 0.3 & 1.21 & 0.249 & 0.083 & 0.027 \\
64 & 100 & 1 & 1.20 & 0.316 & 0.110 & 0.041 \\
64 & 300 & 0.1 & 1.21 & 0.241 & 0.067 & 0.017 \\
64 & 300 & 0.3 & 1.20 & 0.270 & 0.079 & 0.019 \\
64 & 300 & 1 & 1.22 & 0.271 & 0.090 & 0.030 \\
64 & 680 & 1 & 1.24 & 0.263 & 0.090 & 0.034 \\
64 & 1000 & 0.1 & 1.27 & 0.276 & 0.094 & 0.028 \\
64 & 1000 & 0.3 & 1.20 & 0.232 & 0.081 & 0.018 \\
768 & 300 & 0.3 & 1.26 & 0.223 & 0.070 & 0.020 \\
768 & 300 & 1 & 1.31 & 0.278 & 0.100 & 0.040 \\
64 & 1000$^{\dagger}$ & 1 & 1.92 & 0.373 & 0.218 & 0.120 \\
\bottomrule
\end{tabular}
\end{table}

\subsection{E1: two regimes, one geometry}\label{sec:results-e1}
\begin{figure*}[!t]
\centering
\includegraphics[width=\textwidth]{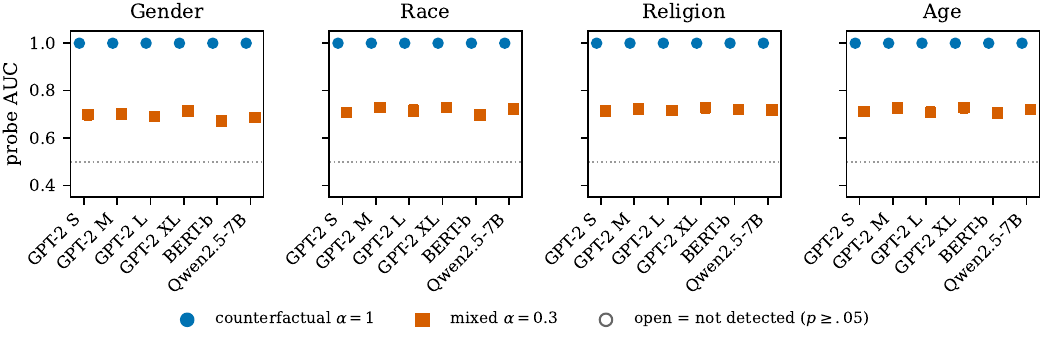}
\caption{E1 at a glance: probe AUC per model and bias dimension in the counterfactual ($\alpha{=}1$, circles) and mixed ($\alpha{=}0.3$, squares) regimes, with across-seed range intervals; open markers would indicate non-detection (none occur). The counterfactual row sits indistinguishably at the ceiling for all $24$ pairs; the mixed regime separates from it by $0.27$--$0.33$ AUC. }
\label{fig:main}
\end{figure*}

Table~\ref{tab:e1_main} and Fig.~\ref{fig:main} are the paper's main measurement. Three facts are stable across all six models and four dimensions. \emph{(i) The counterfactual regime is saturated.} All $24$ model--dimension pairs probe at AUC $=1.000\pm0.000$ (three-decimal rounding) at $\alpha=1$, with cross-fitted separations $\hat s\approx8$--$40$. By \eqref{eq:mixture}, any $s\gtrsim4$ produces this ceiling: counterfactual AUC certifies the presence of surface-marker information but cannot rank models or dimensions---GPT-2~XL at $\hat s\approx40$ and Qwen2.5-7B at $\hat s\approx8$ receive the identical score. \emph{(ii) The mixed regime is informative.} At $\alpha=0.3$ and $n=300$, AUC falls between $0.67$ and $0.73$. Every pair is detected, but all remain far from the ceiling. The zero-free-parameter Theorem~\ref{thm:mixture} prediction (column ``Purity-law pred.'', $\approx0.755$ for all saturated-$s$ pairs) upper-bounds the empirical values with a residual gap of $0.03$--$0.08$ explained by finite-sample probe estimation: the theory is a maximum over probes, which the learned probe attains only as $n\to\infty$. \emph{(iii) The geometry is low-dimensional and mostly offset-clean.} Intrinsic dimensions are $\hat\dB\approx4$--$7$. Measured neutral offsets are small relative to the class separation ($b_{w}<\hat s/2$ in $22$ of $24$ pairs; median ratio $b_w/(\hat s/2)\approx0.26$), so the midpoint model of Assumption~\ref{ass:model} is a good approximation. The two exceptions are the GPT-2 Medium and Large \emph{religion} pairs (ratios $\approx1.1$), whose neutral pool sits beyond the class midpoint along the discriminant. By Corollary~\ref{cor:offset} the purity law remains monotone under any offset, and the offset perturbs the two middle mixture terms in opposite directions, which partially cancels: indeed these two pairs show no elevated prediction error in E2.

With the leakage-controlled protocol and calibrated nulls, gender at $\alpha=0.3$, $n=300$ is detectably above chance in all six models, and \emph{no} below-chance AUC appears anywhere outside the undetectable phase (where AUC $\approx0.5\pm$ noise is expected).

\subsection{E2: purity curves with no free parameters}\label{sec:results-e2}
For each model--dimension pair, $\hat s$ is measured once on the $\alpha=1$ marked set; equation~\eqref{eq:mixture} then predicts the entire AUC--$\alpha$ curve. Fig.~\ref{fig:alpha} shows the curves; Table~\ref{tab:e2_mae} reports the mean absolute error of the zero-parameter prediction and of a one-parameter fitted-$s$ reference, for all $24$ pairs. We highlight three points. First, the empirical curves are monotone in $\alpha$ within seed noise in all $24$ combinations ($20$ are exactly nondecreasing; the other four dip by at most $0.010$ AUC between $\alpha=0.02$ and $0.05$, within one seed s.d.)---the theory's central qualitative prediction, with no trace of an interior optimum. Second, the zero-parameter prediction achieves median MAE $0.015$ (max $0.048$); the one-parameter fit improves this to median $0.010$. Because $\hat s\gg4$ throughout, the predicted curves collapse onto the saturated form $\tfrac12+\tfrac{\alpha(2-\alpha)}{2}$: at high surface separation, purity---not geometry---is the binding constraint, and the collapse is real in the data. Third, the prediction error concentrates in BERT-base ($0.025$--$0.048$ across its four dimensions): the encoder's learned probes sit furthest below the theoretical maximum at this $n$, consistent with the bound interpretation (the theory caps, and finite-sample training determines how closely the cap is approached). The two large-offset religion pairs show no elevated error. One protocol caveat: each curve point reuses a single dataset draw across the five CV seeds, so its error bar reflects re-splitting only; at the smallest purities ($\alpha\le0.05$, i.e.\ $\le8$ marked pairs at $n=300$) dataset-level fluctuation, which these bars do not capture, can place individual points slightly above the population bound.

\emph{Detection accounting under the calibrated test} (Section~\ref{sec:results-e2x}\label{sec:results-e2x}). Of the $216$ curve points, $162$ are detections under the calibrated $1.4\sigma_0$ test; recalibrating from the raw $\sigma_{0}$ removes $14$ nominal detections, \emph{all} at $\alpha\in\{0.02,0.05,0.1\}$---the near-frontier region where Theorem~\ref{thm:detect} predicts marginal power. Under family-wise correction, $159$ of the $162$ survive Benjamini--Hochberg at FDR $.05$, and $109$ survive the (deliberately punitive for $216$ dependent tests) Bonferroni threshold; no qualitative claim of this section rests on an individual near-threshold point. In family F1 (E1, all model--dimension--regime cells) every detection survives Bonferroni outright (largest detected $p=1.3\times10^{-4}<.05/60$), as does every E4 layer-scan cell (largest $p=1.6\times10^{-4}$). 

\begin{figure*}[!t]
\centering
\includegraphics[width=0.98\textwidth]{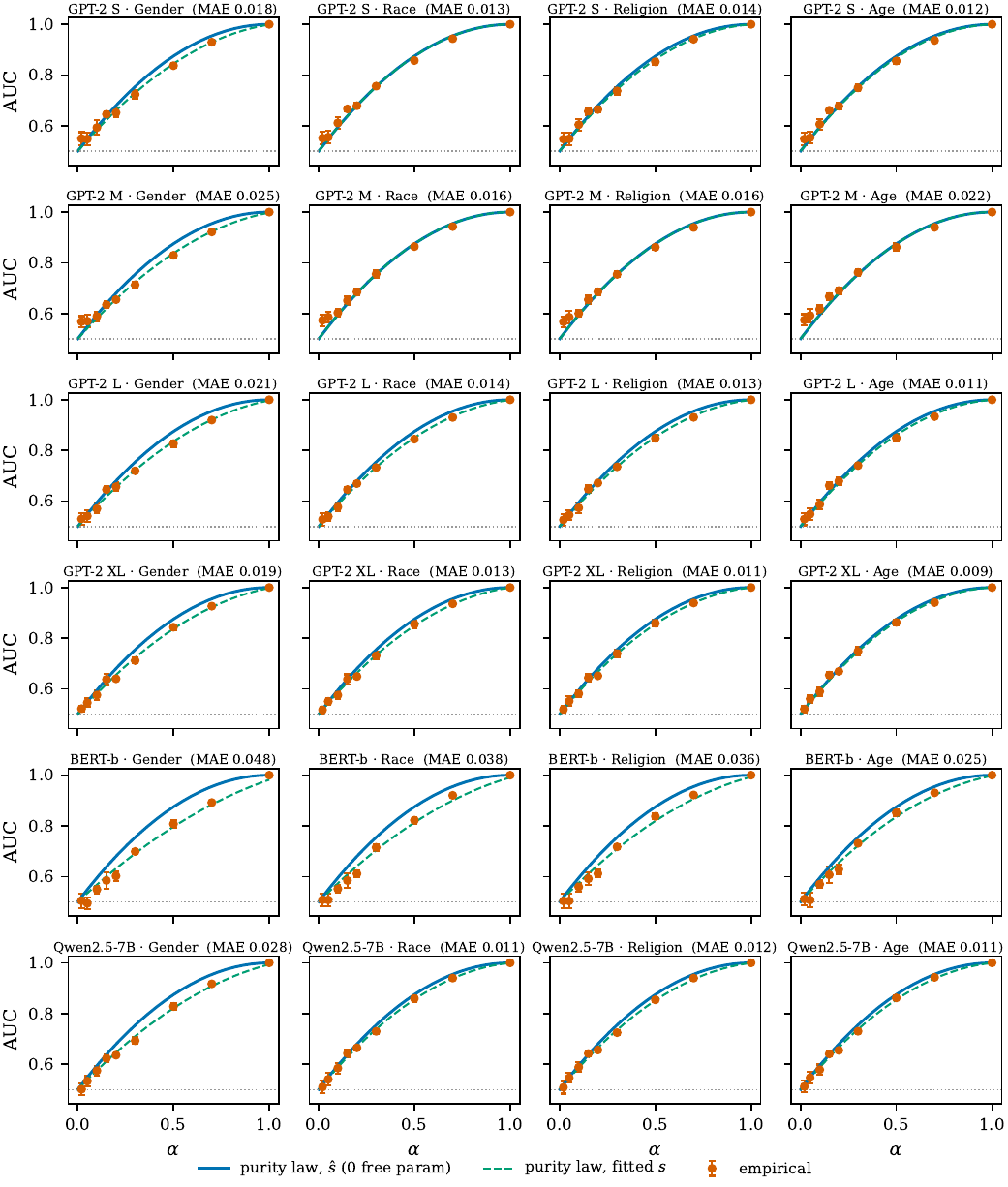}
\caption{E2: empirical AUC($\alpha$) (points, $\pm$ s.d.\ over 5 seeds) vs.\ the zero-free-parameter Theorem~\ref{thm:mixture} prediction (solid; $\hat s$ from the $\alpha{=}1$ set only) and the one-parameter fitted-$s$ curve (dashed), $n=300$. MAEs per panel; all monotone, no interior optimum. }
\label{fig:alpha}
\end{figure*}

\subsection{E2/E3: the detectability phase diagram}\label{sec:results-phase}
Fig.~\ref{fig:phase} overlays the Theorem~\ref{thm:detect} boundary $\astar(n)$ on empirical detection-rate heatmaps in the $(\alpha,n)$ plane. On synthetic data with $s$ known exactly (E3 Part~C), the $50\%$-power contour parallels the predicted boundary with the predicted $-\tfrac12$ log--log slope; the empirical contour sits above $\astar(n)$ by a factor $\approx2$--$2.5$, as it must: $\astar$ is an information-theoretic frontier for the \emph{optimal} probe, while the empirical probe pays an additional training-sample cost in the weak-signal regime. On LLM features (three of the seven combinations run are shown; all appear in the released \texttt{outputs/exp2\_alpha\_sweep.json}), the same structure appears, compressed toward smaller $\alpha$ because $\hat s$ is large: with $s\gtrsim6$ the boundary becomes purity-limited, $\astar(n)\approx z_{.95}/\sqrt{3n}$ independent of $s$---detection at $n=300$ requires $\alpha\gtrsim0.05$--$0.06$ regardless of how separable the markers are. Below the boundary, audits return chance-level AUC \emph{by necessity}, not because models there are unbiased; per-cell detection rates there decay to the cell test's realized level rather than to zero (E9), which is the behavior visible in the lowest rows of Fig.~\ref{fig:phase}.

\begin{figure*}[!t]
\centering
\includegraphics[width=\textwidth]{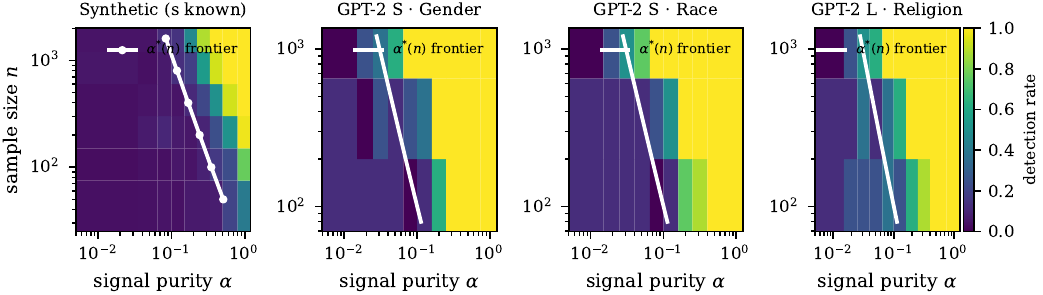}
\caption{Detectability phase diagram: per-cell detection rate (single-seed null test; color) over the $(\alpha,n)$ plane, with the Theorem~\ref{thm:detect} frontier $\astar(n)$ (white). Left: synthetic, $s$ known. Right: LLM model--dimension pairs. The frontier is the optimal-probe $50\%$-power contour: above it, learned-probe detection rates climb toward $1$; below it they decay toward the realized level of the cell test ($\approx0.09$--$0.12$ uncalibrated; E9), never toward power. Cells at $\alpha=1$, $n=1000$ exceed the $680$-sentence marked pool (E9 documents the duplication-inflated null there) and carry no detection claims. }
\label{fig:phase}
\end{figure*}

\subsection{E4: layers}\label{sec:results-e4}
Fig.~\ref{fig:layers} maps probe AUC over every layer of GPT-2 Small and BERT-base in both regimes. In the saturated counterfactual regime, layer choice is nearly irrelevant (AUC $\approx1$ at every depth)---another face of ceiling saturation. In the mixed regime, structure emerges, but not the folklore structure: the \emph{last} layer is never optimal (choosing the best layer instead gains $0.9$--$6.0$ AUC points; more for BERT than GPT-2), yet the optimal depth is model- and dimension-specific and is frequently an \emph{early} layer (layer $1$--$2$ in five of eight combinations). This is coherent with the nature of the signal: marker-based bias evidence is lexically anchored, and surface-lexical information is strongest in early layers. The practical guideline is therefore: \emph{avoid the last layer, and scan layers per model}---the scan costs one forward pass and a handful of probe fits.

\begin{figure}[!t]
\centering
\includegraphics[width=\columnwidth]{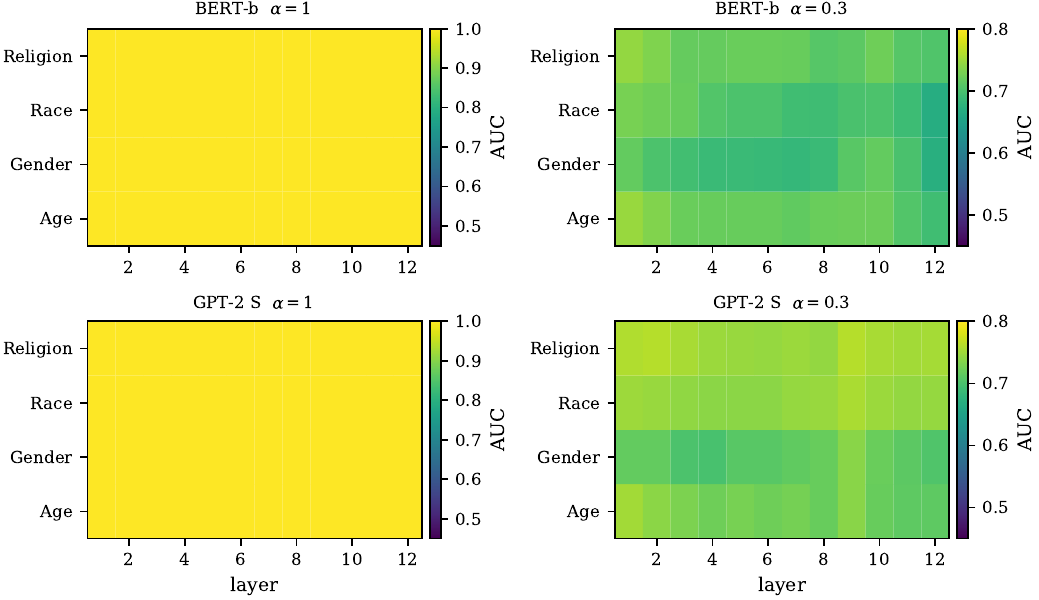}
\caption{E4: layer-wise probe AUC, all layers, both regimes, four bias dimensions. }
\label{fig:layers}
\end{figure}

\subsection{E5: bias--semantic entanglement (C3)}\label{sec:results-e5}
In the synthetic C3 model, learned probes track the optimal closed form $\Phi(s^{*}/\sqrt2)$, $s^{*2}=s^{2}(1-\rho^{2}\eta_{s}^{2}/(1+\eta_{s}^{2}))$, to a maximum error of $0.005$ AUC across $\rho\in[0,0.9]$, while probes \emph{aligned} to the bias axis degrade much faster, following $\Phi(s/\sqrt{2(1+\rho^{2}\eta_{s}^{2})})$ (Fig.~\ref{fig:orth}a)---the quantitative version of C3, and a caution for debiasing methods that operate along a fixed estimated bias direction. On real features in the mixed regime (Fig.~\ref{fig:orth}b), resampling the training set to correlate gender labels with a profession-vs-trait context confound at level $\rho_{\mathrm{lab}}$ and evaluating on balanced held-out data lowers AUC overall, from $0.699$ ($\rho_{\mathrm{lab}}=0$) to $0.615$ ($\rho_{\mathrm{lab}}=0.8$). The intermediate points of this panel come from a single resampling draw each and are visibly noisy (the $\rho_{\mathrm{lab}}=0.2$ point lies \emph{above} baseline), so we read the panel as evidence of an overall transfer penalty, not of a calibrated dose--response curve. The effect is invisible at $\alpha=1$, where the ceiling masks it.

\begin{figure}[!t]
\centering
\includegraphics[width=\columnwidth]{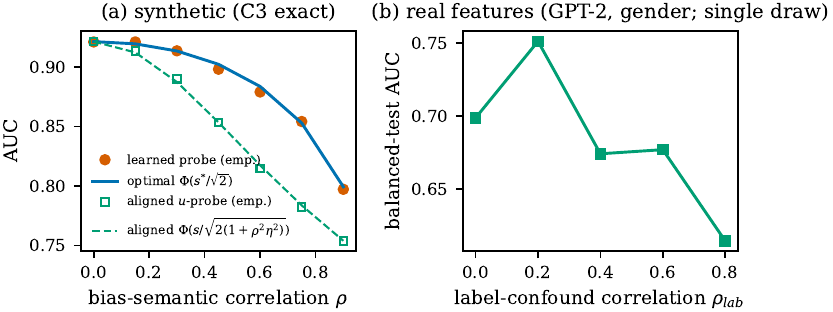}
\caption{E5: orthogonality stress test. (a)~Synthetic: learned probes vs.\ the two closed forms (optimal and axis-aligned). (b)~Real features: balanced-test AUC vs.\ induced label--confound correlation (one resampling draw per point; see text). }
\label{fig:orth}
\end{figure}

\subsection{E6: the generalization envelope}\label{sec:results-e6}
At $\alpha=0.3$ the train--test AUC gap decays with $n$ (Table~\ref{tab:e6_scaling}, Fig.~\ref{fig:scaling}), but the working-protocol probe ($C=1$) decays with log--log slope $\approx-0.09$, far shallower than $-\tfrac12$. This is not a violation of Theorem~\ref{thm:gen}; it is a lesson in reading it. The bound holds for a \emph{fixed} norm budget $\Lambda$, whereas the fitted $\ell_2$-regularized probe spends a larger effective norm at small $n$ (its train AUC pins at $\approx1$ throughout)---the premise, not the inequality, changes with $n$. Constraining the norm ($C=0.01$) steepens the measured decay to slopes $-0.29$ to $-0.37$ with $R^{2}\ge0.97$ (vs.\ $-0.06$ to $-0.12$ at $C=1$), approaching the envelope's rate; the envelope's \emph{level} depends on the unmeasured product $\Lambda\rho_{B}$, so Fig.~\ref{fig:scaling} draws only a slope reference. Two conclusions: the $n^{-1/2}$ statement is an upper envelope, not a rate law for flexible estimators; and scaling must be measured away from the ceiling---at $\alpha=1$ both risks saturate and the gap is identically $\approx0$.

\begin{figure}[!t]
\centering
\includegraphics[width=\columnwidth]{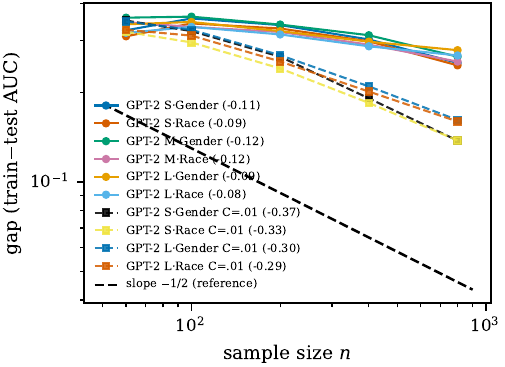}
\caption{E6: generalization gap vs.\ $n$ (log--log), $\alpha=0.3$, with the theoretical $-1/2$ slope for reference. }
\label{fig:scaling}
\end{figure}

\subsection{E7: how far does a label-free diagnostic go?}\label{sec:results-e7}
Fig.~\ref{fig:geometry} closes the loop from geometry to prediction, with one negative and two positive results. \emph{Negative:} the raw concentration index $\betaB$ does not rank real-model probe difficulty within a purity level (per-$\alpha$ Spearman $\rho_s\in[-0.34,+0.36]$, none reliably positive; pooled $\rho_s=0.09$). Real LLM spectra violate the isotropic-noise premise of Lemma~\ref{lem:beta}---dominant non-signal variance directions swamp the top-eigenvalue signal---so the label-free shortcut fails exactly where its model assumption fails, and we report that plainly. \emph{Positive (supervised):} the cross-fitted separation $\hat s$ strongly ranks mixed-regime difficulty. Within purity levels, $\rho_s=0.85$ at $\alpha=0.3$ ($p\approx1.4\times10^{-7}$, $n=24$) and $0.76$ at $\alpha=0.5$ ($p\approx1.8\times10^{-5}$); the full pipeline $\hat s\to$~Theorem~\ref{thm:mixture} tracks empirical AUC across all $72$ (model, dimension, $\alpha$) combinations with $\rho_s=0.97$ and MAE $0.033$ (panel c). \emph{Positive (unsupervised):} routing $\betaB$ \emph{through} the purity model \eqref{eq:betainv} recovers $\rho_s=0.88$ overall---the model supplies the structure that the raw index lacks. ($\betaB$ is computed on standardized features reduced to their top $50$ principal components, and \eqref{eq:betainv} is inverted with $d{=}50$; the raw $d\approx10^{3}$ spectra make the index numerically degenerate. Sample sizes: $n=300$ at $\alpha<1$, $n=600$ at $\alpha=1$.) The net result: $\betaB$ alone, no; $\hat s$ or the model-based pipeline, yes.

\begin{figure*}[!t]
\centering
\includegraphics[width=\textwidth]{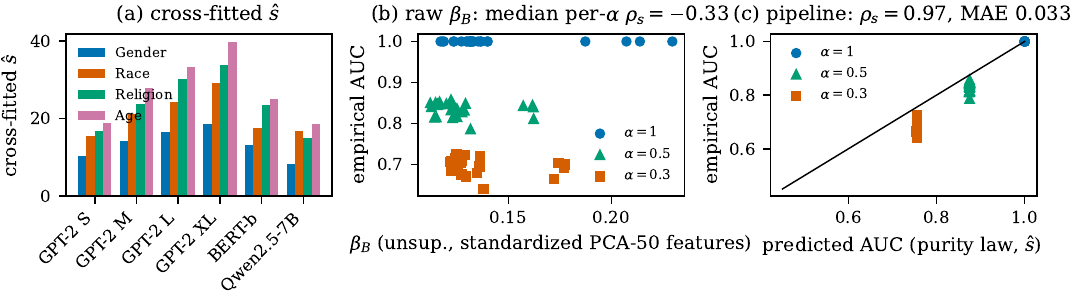}
\caption{E7: from geometry to prediction. (a)~Cross-fitted separations $\hat s$. (b)~Label-free $\betaB$ vs.\ empirical AUC (Spearman $\rho_s$ on panel). (c)~Predicted (Theorem~\ref{thm:mixture} with $\hat s$) vs.\ empirical AUC. }
\label{fig:geometry}
\end{figure*}

\subsection{E1b: naturalistic paraphrases}
We paraphrased template pairs with a locally served instruction model (via \texttt{ollama}), keeping a pair only if both paraphrases lexically preserve their demographic marker ($208$ validated pairs; the strict filter rejects most generations, which we report as a property of the protocol). Because paraphrases have no matched neutral pool, E1b evaluates the counterfactual regime only (Table~\ref{tab:e1_natural}), with two outcomes. First, the ceiling survives de-templating: AUC remains $\approx1.0$, so the saturation phenomenon is not a template artifact. Second, the \emph{effect size does not survive unchanged}: cross-fitted $\hat s$ drops roughly twofold relative to templates (e.g., GPT-2 Small age $18.8\to7.6$; Qwen2.5-7B gender $8.3\to2.8$)---template regularity inflates measured separations even when it does not change the qualitative verdict. At the smallest naturalistic separation the ceiling begins to lift exactly as Lemma~\ref{lem:binormal} dictates: Qwen2.5-7B gender measures AUC $=0.984$ against the binormal prediction $\Phi(2.81/\sqrt2)=0.977$. Audits that report effect sizes from templated counterfactuals should expect naturalistic $\hat s$ to be roughly half as large.

\subsection{E10: nonlinear probes---what they recover, and what they cannot}\label{sec:results-e10}
E10 runs RBF-kernel SVMs and one-hidden-layer MLPs under the identical protocol, on both E3 generators and on the LLM features (Fig.~\ref{fig:nonlinear}. \emph{Purity (S1).} On the flat mixture with known $(s,\alpha)$, no probe exceeds the law \eqref{eq:mixture}: the maximum overshoot over $12$ $(s,\alpha)$ settings is $+0.016$ for the linear probe ($\approx1.2$ Monte-Carlo s.d.\ at $n_{\mathrm{te}}=2000$), $-0.005$ for the RBF-SVM, and $-0.015$ for the MLP---the empirical face of Remark~\ref{rem:bayes}: purity dilution is a Bayes limit, and nonlinearity does not buy it back. \emph{Curvature (S2).} On spheres with ambient noise, in every cell where the chordal--geodesic gap exceeds $0.02$ AUC (all at $\kap=4$), the kernel and MLP probes sit \emph{at or below} the linear probe, and the geodesic oracle---which projects onto the true sphere and scores by exact geodesic distances---matches the linear probe to $\le0.002$ (e.g.\ $0.934$ vs $0.933$ at $r=1.5$): none of the gap is recovered. Ambient noise of scale $\sigma$ against a manifold of radius $1/\sqrt{\kap}$ destroys geodesic information itself, so in this regime the ceiling \eqref{eq:ceiling} in practice binds \emph{all} probes, and the idealized recovery of Corollary~\ref{cor:ceiling} should be read as an upper bound attained only when noise is intrinsic. \emph{Real features (R).} Across $72$ settings ($6$ models $\times$ $4$ dimensions $\times$ $\alpha\in\{0.1,0.3,1\}$, $n=300$), the linear probe dominates: the median paired difference is $-0.068$ AUC for the RBF-SVM and $-0.015$ for the MLP; gains above $+0.005$ occur in $0$ and $4$ of $72$ settings respectively, all four at $\alpha=0.3$ and three of the four on BERT-base---precisely the model whose linear probes sit furthest below the purity-law maximum in E2---with the largest gain $+0.035$ (religion). Calibrated detections drop from $69/72$ (linear) to $46$ (SVM) and $53$ (MLP). The practical reading is sharp: at audit sample sizes, nonlinearity costs more in variance than it recovers in structure, and a nonlinear probe is worth trying only where the linear probe falls measurably short of its purity-law prediction.

\begin{figure}[!t]
\centering
\includegraphics[width=\columnwidth]{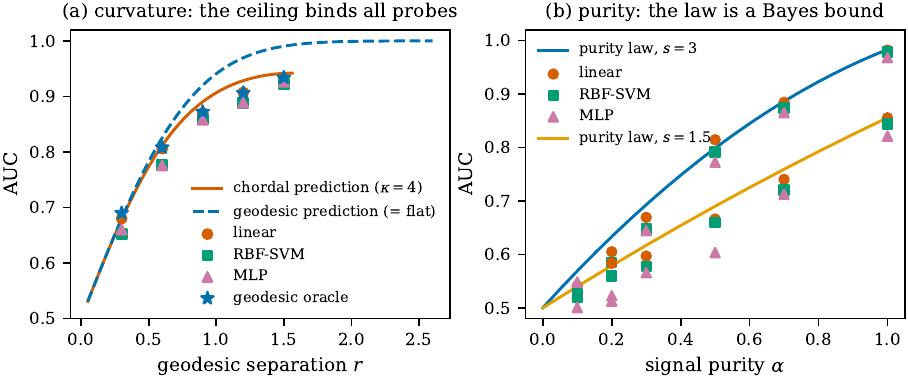}
\caption{E10 on the E3 generators, $n_{\mathrm{te}}=2000$. (a)~Curvature ($\kap=4$, ambient noise): linear, RBF-SVM, MLP, and the geodesic oracle all track the chordal prediction (solid) and saturate together below the idealized geodesic value (dashed). (b)~Purity ($s\in\{1.5,3\}$): all probes sit on or below the law \eqref{eq:mixture}, which is a Bayes bound (Remark~\ref{rem:bayes}). }
\label{fig:nonlinear}
\end{figure}

\subsection{E11: how neutral is ``neutral''? Measured stereotype leakage}\label{sec:results-e11}
The purity model treats neutral texts as label-independent. E11 measures how far that idealization is from the truth of the representations, and how far our protocol is insulated from the difference. \emph{(A) The protocol is clean by construction.} Strong classifiers (linear, RBF-SVM, MLP) trained on neutral-only features against the protocol's randomly assigned labels probe at chance: $0$ calibrated detections in $72$ tests (max $|\auc-\tfrac12|=0.069$), confirming that the random labeling used everywhere in E1--E7 admits no leakage into the reported numbers. \emph{(B) But the text is not semantically neutral.} Assigning each of the $160$ neutral profession sentences its WinoBias occupation-stereotype label \cite{zhao2018winobias} ($16$ professions overlap our inventory; $9$ male-, $7$ female-dominated per BLS), the $\alpha{=}1$ marked-trained gender probe---never shown a neutral sentence---scores them at AUC $0.688$--$0.873$ across the six models (all $p\le2.4\times10^{-5}$), largest for Qwen2.5-7B ($0.873$) and BERT-base ($0.818$); training directly on the neutral sentences reaches AUC $1.000$ (linear, grouped CV) in \emph{every} model. Occupation stereotype is thus fully linearly encoded in ``neutral'' representations, and the demographic-marker direction partially aligns with it. \emph{(C) Leaky purity stress test.} Rebuilding the E2 gender sweep with neutral labels drawn from the stereotype with probability $\pi$ (the protocol is $\pi=\tfrac12$; natural corpora with stereotype-correlated ground truth sit at $\pi>\tfrac12$): at $\pi=\tfrac12$ the curves are monotone within seed noise (max dip $0.043$, ${\approx}1.8$ seed s.d.); at $\pi=0.75$ the small-$\alpha$ tail \emph{lifts} to a leakage floor of $0.64$--$0.69$ (instead of decaying to $\tfrac12$) and the curve is genuinely non-monotone (max dip $0.081$, beyond seed noise)---as it must be, since the theorem's label-independence premise is violated; at $\pi=1$ neutrals are simply labeled data and the curve pins at AUC $\approx1$ for all $\alpha$. The practical translation: on corpora whose latent labels correlate with stereotypes, $\alpha$ measured by \emph{marker} counting understates the effective signal, a null probe result is even harder to interpret, and the floor of the AUC--$\alpha$ curve estimates the leakage strength; our marker-based protocol sidesteps this by randomization, at the cost of modeling only marker-borne signal.

\section{Discussion}\label{sec:discussion}

\subsection{Guidelines for practitioners}
\emph{(1) Treat audits as power analyses.} Before probing, fix the purity $\alpha$ of the audit corpus (measurable by counting marked texts) and the minimal effect size $s_{\min}$ of interest; then budget $n_{\mathrm{te}}\ge(z_{1-\delta}+z_{1-\gamma})^{2}\big/\big(3\,[\aucmax(\alpha;s_{\min})-\tfrac12]^{2}\big)$ via \eqref{eq:detcond} with $\sigma_{0}\approx1/\sqrt{3n_{\mathrm{te}}}$, or read $n(\alpha)$ off the phase diagram. If the audit statistic is a pooled cross-validated AUC rather than a single held-out split, multiply $\sigma_{0}$ by the E9 calibration factor ($1.4$ under our protocol; budget $\times1.96$)---or simply use one held-out split, for which $\sigma_{0}$ is exact. A null result below the frontier should be reported as \emph{underpowered}. \emph{(2) Do not rank models by counterfactual AUC.} Above $s\approx4$ the ceiling hides all differences; rank by $\hat s$ (cross-fitted) or by mixed-regime AUC at matched $(\alpha,n)$. \emph{(3) Control template leakage---and measure it.} We compare grouped (context-disjoint) with random cross-validation on identical data (E8, GPT-2 Small and Large, all four dimensions): with our design, in which marked and neutral prompts share one context grid, random CV inflates AUC by at most $0.011$ ($\alpha=0.3$; $<10^{-4}$ at the $\alpha=1$ ceiling). The inflation is small \emph{because} the design crosses markers with contexts; with organic corpora no such guarantee exists, so grouped CV remains the conservative default and the grouped--random gap is itself a cheap leakage meter. \emph{(4) Never the last layer---and scan.} Away from saturation the last layer is uniformly suboptimal (E4), but the best depth is model- and dimension-specific and often early; a per-model layer scan costs one forward pass and recovers up to $6$ AUC points. \emph{(5) Prefer cross-fitted effect sizes.} Naive Mahalanobis plug-ins inflate catastrophically at $d\gtrsim n$ (E3-D); any pipeline reporting raw separations at $d/n>1$ without cross-fitting or shrinkage is suspect. \emph{(6) Do not trust raw $\betaB$ on real features.} The label-free index ranks difficulty only when routed through the purity model \eqref{eq:betainv} (E7: pipeline $\rho_s=0.88$ vs.\ raw $\rho_s\approx0$); with even a small labeled counterfactual set, the cross-fitted $\hat s$ is the reliable diagnostic ($\rho_s=0.85$/$0.76$ within the mixed regime at $\alpha=0.3$/$0.5$). \emph{(7) Do not reach for nonlinear probes to rescue a weak audit.} Against purity dilution no classifier can beat \eqref{eq:mixture} (Remark~\ref{rem:bayes}); under ambient noise the curvature ceiling binds nonlinear probes too (E10-S2); and at audit-scale $n$ the kernel and MLP baselines lose more to variance than they recover (median $-0.068$/$-0.015$ AUC, E10-R). Their legitimate use is diagnostic: a nonlinear probe that beats the linear one signals that the linear probe sits below its purity-law prediction, as in the BERT-base cases of E10.

\subsection{Limitations}
\emph{Model of neutrality.} Theorem~\ref{thm:mixture} treats neutral texts as label-independent with a shared covariance. E11 measures how idealized that is: occupation-stereotype structure is fully linearly decodable from our ``neutral'' sentences (AUC $1.0$ direct; $0.69$--$0.87$ zero-shot from the marker direction), our protocol is insulated from it only because neutral labels are randomized, and under stereotype-correlated labeling ($\pi>\tfrac12$) the purity curve acquires a leakage floor and loses monotonicity---so the purity law should be applied to natural corpora only together with a leakage measurement of the E11-B type. Corollary~\ref{cor:offset} handles mean offsets; a full treatment of graded purity ($\alpha$ as a distribution, leakage as a second signal component) is future work. \emph{Curvature on LLM features.} Theorem~\ref{thm:curvature} is exact on space forms; on LLM features our sagitta estimator carries $20$--$30\%$ bias (E3-D), so we use curvature as a bounded correction and a qualitative mechanism, not as a precisely measured quantity. \emph{Template ecology.} Even with paraphrase robustness checks, template-derived prompts undercover the diversity of deployment text; purity in the wild also varies by domain. \emph{Scope of ``bias.''} Linear detectability of demographic markers is evidence about representations, not about downstream harms; bridging to behavioral bias metrics \cite{goldfarb2021intrinsic,delobelle2022measuring} remains open. Each dimension is moreover operationalized as a single binary marker contrast (man/woman, White/Black, Christian/Muslim, young/elderly) in English; multi-group, intersectional, non-lexical, and cross-lingual operationalizations are untested.

\section{Conclusion}
Throughout this paper ``probe failure'' has meant one operational thing: under the leakage-controlled protocol of Section~\ref{sec:exp-design}, the level-$.05$ pooled-OOF detection test returns a null, or the measured AUC falls short of the theory's maximum by more than the calibrated error of the pipeline. In that sense linear bias probes do not fail mysteriously. They fail for three quantifiable reasons---not enough marked signal ($\alpha$ below the detectability frontier), not enough samples ($n$ below the $\sqrt{\dB/n}$ budget), or geometry that linear readout cannot exploit (chordal contraction under curvature)---and they succeed trivially in the counterfactual regime because surface markers put $s$ deep into the saturated zone of the purity law. Two boundaries of this account are now measured rather than conjectured. Nonlinear probes rescue neither mechanism: the purity law caps them provably (Remark~\ref{rem:bayes}, confirmed in E10-S1), and under ambient noise even a geodesic oracle recovers none of the curvature-induced loss, while kernel and MLP probes trail the linear probe on real features at audit-scale $n$ (E10)---so the linear theory's scope is, in practice, the scope of probing tout court in our regimes. And the neutrality idealization is measurably false of the text (stereotype signal is linearly decodable from ``neutral'' sentences, E11) yet harmless to our protocol by randomization, with the leaky-purity stress test mapping how it distorts audits on natural corpora. Under the stated assumptions---and with the detection test's realized level verified on exact nulls (E9)---the combined pipeline predicts real probe behavior across six models and four bias dimensions from measured geometry, in the central case with no parameters fitted to the predicted curves. The practical payoff is a change of frame: a bias audit is a statistical power problem, and it should be designed---and its null results reported---accordingly.

\bibliographystyle{IEEEtran}
\bibliography{refs}

\clearpage
\appendices

\section{Proofs for Section~\ref{sec:theory-gen}}\label{app:gen}

\subsection{Proof of Theorem~\ref{thm:gen}}
\begin{proof}
Write $\tilde h=h-c$ where $c$ is the center of the smallest enclosing Euclidean ball of $B$, so $\|\tilde h\|\le\rho_{B}$ for all $h\in B$. For the class $\mathcal H_{\Lambda}=\{h\mapsto w^{\top}\tilde h:\|w\|\le\Lambda\}$, the empirical Rademacher complexity obeys
\[
\widehat{\mathfrak R}_{S}(\mathcal H_{\Lambda})
=\E_{\sigma}\sup_{\|w\|\le\Lambda}\Big\langle w,\tfrac1n\sum_{i}\sigma_{i}\tilde h_{i}\Big\rangle
=\frac{\Lambda}{n}\,\E_{\sigma}\Big\|\sum_{i}\sigma_{i}\tilde h_{i}\Big\|.
\]
By Jensen, $\E_{\sigma}\|\sum_i\sigma_i\tilde h_i\|\le(\E_{\sigma}\|\sum_i\sigma_i\tilde h_i\|^{2})^{1/2}=(\sum_{i}\|\tilde h_{i}\|^{2})^{1/2}\le\rho_{B}\sqrt n$, whence $\widehat{\mathfrak R}_{S}\le\Lambda\rho_{B}/\sqrt n$. The unit-margin ramp $\psi(t)=\min(1,\max(0,1-t))$ is $1$-Lipschitz and $\psi(y\,f(h))\ge\mathbf 1[y f(h)\le0]$, so by Talagrand's contraction lemma the loss class has Rademacher complexity at most $\Lambda\rho_{B}/\sqrt n$, and the standard Rademacher generalization theorem \cite{bartlett2002rademacher} gives, with probability $1-\delta$, simultaneously for all $h\in\mathcal H_\Lambda$:
\[
R_{0\text{-}1}(h)\le \E\,\psi \le \widehat R_{n}(h)+\frac{2\Lambda\rho_{B}}{\sqrt n}+3\sqrt{\frac{\log(2/\delta)}{2n}},
\]
where $\widehat R_n$ is the empirical ramp risk and the deviation constant $3$ covers the two-sided bounded-difference step for the empirical complexity.

For \eqref{eq:radius}: any two points of $B$ at geodesic distance $r$ have Euclidean (chordal) distance at most $r$ (a straight segment is no longer than any connecting path), so $D_{B}^{\mathrm{ch}}\le D_{B}$. Since $B$ is closed with $K\ge\kap_{\min}>0$, the Bonnet--Myers theorem \cite{docarmo1992riemannian} gives $D_{B}\le\pi/\sqrt{\kap_{\min}}$. Jung's theorem in the dimension-free form $\rho\le D^{\mathrm{ch}}/\sqrt2$ completes \eqref{eq:radius}.
\end{proof}

\begin{remark}[Sharpness; no chordal contraction off space forms]\label{rem:needle}
The constant $\pi$ in \eqref{eq:radius} cannot be improved to the spherical value $2$. Consider the surface of revolution of the profile $\rho(z)=a\cos(\sqrt{\kap}\,z)$, $|z|\le\pi/(2\sqrt{\kap})$, with its two tips smoothed: its Gaussian curvature is $K=\kap/(1+\rho'(z)^{2})^{2}\ge\kap(1-a^{2}\kap)$ everywhere, yet its pole-to-pole \emph{chord} tends to $\pi/\sqrt{\kap}$ as $a\to0$ (numerically, $K\ge0.995\,\kap$ with chord $3.138/\sqrt{\kap}$ already at $a=0.05$). Consequently the spherical chord bound $(2/\sqrt{\kap})\sin(\sqrt{\kap}\,D_{B}/2)$, exact on round spheres, is \emph{false} for general submanifolds with $K\ge\kap$: extrinsic chords are not controlled by intrinsic comparison beyond chord $\le$ geodesic. (Toponogov-type comparison controls intrinsic, not extrinsic, distances.) The constant $2$ is recovered under an extrinsic hypothesis---all normal curvatures $\ge\sqrt{\kap}$, by Blaschke's rolling theorem---or on the space forms of Theorem~\ref{thm:curvature}.
\end{remark}

\subsection{Proof of Proposition~\ref{prop:lower}}
\begin{proof}
Fix $\dB\le n/c_{1}$ and work in $\R^{\dB}$ (embed as an affine slice for the general case). For $\sigma\in\{\pm1\}^{\dB}$ define class-conditional laws $P_{\sigma}^{\pm}=\mathcal N(\pm\mu_{\sigma}/2, I)$ with
\[
\mu_{\sigma}=s\,\frac{\sigma}{\sqrt{\dB}},\qquad \|\mu_{\sigma}\|=s,\qquad s\le1 \text{ chosen below},
\]
labels balanced. The Bayes rule for $P_\sigma$ is the linear rule with direction $\sigma/\sqrt{\dB}$ and $R^{*}=\Phi(-s/2)$, identical for all $\sigma$.

\emph{Step 1 (risk lower bound via direction error).}
For $0$-$1$ loss, excess risk is the margin-weighted disagreement with the Bayes rule: $R_{\sigma}(\hat g)-R^{*}=\int_{\{\hat g\ne g^{*}_{\sigma}\}}|2\eta_{\sigma}-1|\,dP^{X}_{\sigma}=:d_{\sigma}(\hat g,g^{*}_{\sigma})$, where $\eta_\sigma$ is the regression function. For $\sigma$ and $\sigma^{(j)}$ differing only in coordinate $j$, the Bayes rules are the halfspaces normal to $u_{\sigma}=\sigma/\sqrt{\dB}$ and $u_{\sigma^{(j)}}$, whose angle $\theta_{j}$ satisfies $\cos\theta_{j}=1-2/\dB$, i.e.\ $\sin^{2}(\theta_{j}/2)=1/\dB$ and $\theta_{j}^{2}\ge4/\dB$. A Gaussian computation gives, for $\theta_{j}\le1$ (i.e.\ $\dB\ge5$; smaller $\dB$ is absorbed into $c_{0}$),
\[
\begin{aligned}
d_{\sigma}\big(g^{*}_{\sigma^{(j)}},g^{*}_{\sigma}\big)
&\ge \Phi\big(-\tfrac s2\cos\theta_{j}\big)-\Phi\big(-\tfrac s2\big)\\
&\ge \frac{s}{8}\,\phi\!\Big(\frac s2\Big)\theta_{j}^{2}
 \ge \frac{s\,\phi(s/2)}{2\dB},
\end{aligned}
\]
using $1-\cos\theta\ge\theta^{2}/4$ on $[0,1]$ and the mean value theorem. Because the two models' marginals and margins agree up to a factor $\ge\tfrac12$ on the disagreement region (their class-conditional means differ by $2s/\sqrt{\dB}\le2/\sqrt{\dB}$), the triangle inequality for the pseudo-distance $d$ \cite[Ch.~2]{tsybakov2009introduction} yields, for \emph{any} classifier $\hat g$,
\[
R_{\sigma}(\hat g)+R_{\sigma^{(j)}}(\hat g)-2R^{*}\;\ge\;\tfrac12\,d_{\sigma}\big(g^{*}_{\sigma},g^{*}_{\sigma^{(j)}}\big)\;\ge\;\frac{s\,\phi(s/2)}{4\dB},
\]
i.e.\ a per-coordinate half-sum excess of at least $s\,\phi(s/2)/(8\dB)$.

\emph{Step 2 (Assouad).}
By Assouad's lemma \cite[Thm.~2.12]{tsybakov2009introduction}, for the sup over the hypercube,
\[
\max_{\sigma}\E\big[R_{\sigma}(\mathcal A)-R^{*}\big]
\ge \frac{\dB}{2}\cdot\frac{s\,\phi(s/2)}{8\dB}\cdot\Big(1-\sqrt{\mathrm{KL}_{j}\,n/2}\Big),
\]
where $\mathrm{KL}_{j}$ is the per-sample KL divergence between the mixtures indexed by $\sigma$ and $\sigma^{(j)}$. Since the class-conditionals are unit-variance Gaussians whose means differ by $\|\mu_{\sigma}-\mu_{\sigma^{(j)}}\|=2s/\sqrt{\dB}$ in one coordinate and the label is balanced, $\mathrm{KL}_{j}\le \frac12\big(\tfrac{2s}{\sqrt{\dB}}\big)^{2}\cdot\frac14\cdot 2=\frac{s^{2}}{\dB}$.

\emph{Step 3 (calibration).}
Choose $s=\tfrac12\sqrt{\dB/n}\ (\le1$ for $n\ge\dB/4$). Then $\mathrm{KL}_{j}\,n/2= s^{2}n/(2\dB)=1/8$ and $1-\sqrt{1/8}\ge0.64$, so
\[
\max_{\sigma}\E\big[R_{\sigma}(\mathcal A)-R^{*}\big]
\;\ge\;0.64\cdot\frac{s\,\phi(s/2)}{16}
\;\ge\;c_{0}\sqrt{\dB/n},
\]
with $c_{0}=0.64\,\phi(1/2)/32\ge 5.6\times10^{-3}$ absorbing constants ($\phi(s/2)\ge\phi(1/2)$ for $s\le1$). For $n<c_1 \dB$ the trivial bound $\min\{1,\cdot\}$ applies. All quantities are nonnegative throughout.
\end{proof}

\section{Proof of Theorem~\ref{thm:conditions}}\label{app:cond}
\begin{proof}
\textbf{C1.} Immediate from Corollary~\ref{cor:ceiling}: if $\auc\ge1-\epsilon$ then $\Phi\big(\sqrt2/(\sqrt\kap\,\sigma_w)\big)\ge1-\epsilon-O(\eta^{2})$, i.e. $\sqrt2/(\sqrt\kap\,\sigma_w)\ge z_{\epsilon}$ up to the stated correction, which rearranges to C1.

\textbf{C2.} If the audit certifies risk $\epsilon$ from data, the uniform bound \eqref{eq:genbound} must be $\le\epsilon$, giving the first display; necessity of $n\gtrsim\dB/\epsilon^{2}$ follows from Proposition~\ref{prop:lower} (any rule with expected excess $\le\epsilon$ on all instances needs $c_{0}\sqrt{\dB/n}\le\epsilon$). The test-power budget is Theorem~\ref{thm:detect}.

\textbf{C3.} Model: $X\mid y=\pm\sim\mathcal N(\pm\tfrac s2 u+\eta_{s}\zeta v,\;I)$ marginally over the independent semantic factor $\zeta\sim\mathcal N(0,1)$, with $\|u\|=\|v\|=1$, $u^{\top}v=\rho$. Then $\Sigma=I+\eta_{s}^{2}vv^{\top}$ and $\Delta=s\,u$. Sherman--Morrison:
$\Sigma^{-1}=I-\frac{\eta_{s}^{2}}{1+\eta_{s}^{2}}vv^{\top}$, so
\[
s^{*2}=\Delta^{\top}\Sigma^{-1}\Delta
=s^{2}\Big(1-\frac{\rho^{2}\eta_{s}^{2}}{1+\eta_{s}^{2}}\Big),
\]
and by Lemma~\ref{lem:binormal} the optimal linear probe attains exactly $\Phi(s^{*}/\sqrt2)$; as $\eta_{s}\to\infty$, $s^{*}\to s\sqrt{1-\rho^{2}}$. Requiring $\Phi(s^{*}/\sqrt2)\ge1-\epsilon$ gives $s^{*2}\ge2z_{\epsilon}^{2}$, i.e.\ $\rho^{2}\frac{\eta_{s}^{2}}{1+\eta_{s}^{2}}\le1-\frac{2z_{\epsilon}^{2}}{s^{2}}$, which rearranges to the stated bound $\rho^{2}\le(1+\eta_{s}^{-2})(1-2z_{\epsilon}^{2}/s^{2})_{+}$ (and to $\rho^{2}\le1-2z_{\epsilon}^{2}/s^{2}$ in the limit). For the aligned probe $w=u$: the score variance is $u^{\top}\Sigma u=1+\rho^{2}\eta_{s}^{2}$ while the mean gap is $s$, giving $\Phi\big(s/\sqrt{2(1+\rho^{2}\eta_{s}^{2})}\big)$ by \eqref{eq:binormal}.
\end{proof}

\section{Proofs of Lemma~\ref{lem:binormal} and Theorem~\ref{thm:curvature}}\label{app:curv}

\subsection{Proof of Lemma~\ref{lem:binormal}}
\begin{proof}
For fixed $w$, $S_{\pm}=w^{\top}X|y=\pm$ are independent Gaussians with means $w^{\top}\mu_{\pm}$ and common variance $w^{\top}\Sigma w$. Hence $S_{+}-S_{-}\sim\mathcal N(w^{\top}\Delta,\,2w^{\top}\Sigma w)$ and $\auc(w)=\Prob(S_{+}>S_{-})=\Phi\big(w^{\top}\Delta/\sqrt{2w^{\top}\Sigma w}\big)$. Writing $v=\Sigma^{1/2}w$, the argument is $\langle v,\Sigma^{-1/2}\Delta\rangle/(\sqrt2\|v\|)\le\|\Sigma^{-1/2}\Delta\|/\sqrt2=s/\sqrt2$ by Cauchy--Schwarz, with equality iff $v\parallel\Sigma^{-1/2}\Delta$, i.e.\ $w\propto\Sigma^{-1}\Delta$. For the final claim: the log-likelihood ratio between the two classes is an affine function of $w^{*\top}h$, so by Neyman--Pearson every point of the optimal ROC curve is attained by thresholding $w^{*\top}h$; the Bayes-ROC AUC therefore equals the linear maximum $\Phi(s/\sqrt2)$ \cite{hanley1982meaning}.
\end{proof}

\subsection{Proof of Theorem~\ref{thm:curvature}}
\begin{proof}
Work on the sphere $S^{\dB}(R)\subset\R^{\dB+1}\subset\R^{d}$, $R=1/\sqrt{\kap}$, and let $p_{\pm}$ have geodesic distance $r=R\,\vartheta$, $\vartheta\in(0,\pi]$. Choose coordinates with $p_{\pm}=R(\cos\tfrac\vartheta2,\,\pm\sin\tfrac\vartheta2,0,\dots,0)$, and let $u=(p_{+}-p_{-})/\|p_{+}-p_{-}\|=e_{2}$ denote the chord direction; $\|p_{+}-p_{-}\|=2R\sin\tfrac\vartheta2=c(r)$, which proves the chord formula and $c(r)=g(\kap,r)r$.

A sample from class $y$ is $X=\exp_{p_{y}}(\tau\xi)+\sigma\varepsilon$ with $\xi\sim\mathcal N(0,I_{\dB})$ in $T_{p_y}S^{\dB}(R)$. Writing $\theta=\tau\|\xi\|/R$ for the angular displacement, the exponential map gives
\[
\exp_{p_y}(\tau\xi)=\cos(\theta)\,p_{y}+R\sin(\theta)\,\hat\xi,\qquad \hat\xi=\xi/\|\xi\|\in T_{p_y}.
\]
Let $\eta=\tau\sqrt{\dB}/R=\tau\sqrt{\dB}\sqrt{\kap}$; then $\theta=O_{P}(\eta)$. Expanding to second order, $\exp_{p_y}(\tau\xi)=p_{y}+\tau\xi-\frac{\theta^{2}}{2}p_{y}+O(R\theta^{3})$.

\emph{Mean gap along $u$.} $\E[\exp_{p_\pm}(\tau\xi)]=\E[\cos\theta]\,p_{\pm}$ by symmetry ($\E[\sin\theta\,\hat\xi]=0$), and $\E\cos\theta=1-O(\eta^{2})$. Hence $\E[X|+]-\E[X|-]=(1-O(\eta^{2}))\,(p_{+}-p_{-})$: the mean separation is $c(r)(1-O(\eta^{2}))$, entirely along $u$.

\emph{Variance along $u$.} The tangent space at $p_{\pm}$ decomposes as $T_{p_\pm}=\mathrm{span}(t_{\pm})\oplus W$, where $t_{\pm}$ is the great-circle direction toward the other center and $W=\mathrm{span}(e_{3},\dots,e_{\dB+1})$ is common to both and orthogonal to $u$. Explicitly $t_{\pm}=(\mp\sin\tfrac\vartheta2,\,\cos\tfrac\vartheta2,0,\dots)$, so $\langle t_{\pm},u\rangle=\cos\tfrac\vartheta2$. To first order in $\eta$ the fluctuation of $X$ along $u$ is $\tau\,\xi_{1}\langle t_{y},u\rangle+\sigma\,\varepsilon_{u}$ with $\xi_1\sim\mathcal N(0,1)$; the radial (second-order) term contributes $O(\eta^{2})$ variance. Hence
$\operatorname{Var}(X_{u}\,|\,y)=\sigma^{2}+\tau^{2}\cos^{2}\!\big(\tfrac{\sqrt\kap\,r}{2}\big)+O(\eta^{2})=\sigma_{w}^{2}(1+O(\eta^{2}))$.

\emph{Optimality of $u$ and the AUC.} To first order the class-conditional laws are Gaussian with common covariance $\Sigma_{0}=\sigma^{2}I+\tau^{2}\Pi_{T}$ ($\Pi_T$ the tangent projector) and mean gap $c(r)\,u$. Since $u$ is an eigenvector direction decomposition-wise ($\Sigma_{0}u=\sigma_{w}^{2}u+O(\eta)$ cross terms), Lemma~\ref{lem:binormal} gives $\aucmax^{\mathrm{lin}}=\Phi\big(c(r)/(\sqrt2\,\sigma_{w})\big)+O(\eta^{2})$: the Gaussian approximation error and the neglected quadratic terms each perturb the AUC by $O(\eta^{2})$ (AUC is a bounded functional with bounded derivative in the mean/variance parameters here). This is \eqref{eq:curvauc}; the ceiling follows since $\sin\le1$ gives $c(r)\le2/\sqrt\kap$, and $c(r)\le r$ since $\sin x\le x$.
\end{proof}

\section{Proofs for Theorems~\ref{thm:mixture} and \ref{thm:detect}}\label{app:mix}

\subsection{Proof of Theorem~\ref{thm:mixture}}
\begin{proof}
Fix $w$ with $s_{w}=w^{\top}\Delta/\sqrt{w^{\top}\Sigma w}$ and standardize scores by $\sqrt{w^{\top}\Sigma w}$. Conditional on the independent marked/neutral indicators $(M_{+},M_{-})\in\{m,u\}^{2}$ of the two test points in $\auc=\Prob(S_{+}>S_{-})$, the score pairs are independent Gaussians with unit variance and means: $(+\tfrac{s_w}{2},-\tfrac{s_w}{2})$ on $\{m,m\}$ (prob.\ $\alpha^{2}$); $(+\tfrac{s_w}{2},0)$ on $\{m,u\}$; $(0,-\tfrac{s_w}{2})$ on $\{u,m\}$ (each prob.\ $\alpha(1-\alpha)$); $(0,0)$ on $\{u,u\}$ (prob.\ $(1-\alpha)^{2}$). Using $\Prob(\mathcal N(a,1)>\mathcal N(b,1))=\Phi((a-b)/\sqrt2)$:
\[
\auc(w)=\alpha^{2}\Phi\big(\tfrac{s_{w}}{\sqrt2}\big)+2\alpha(1-\alpha)\Phi\big(\tfrac{s_{w}}{2\sqrt2}\big)+\tfrac{(1-\alpha)^{2}}{2}.
\]
Each term is nondecreasing in $s_{w}$, and the first two strictly increasing, so the maximizer over $w$ is that of $s_{w}$, namely $w^{*}\propto\Sigma^{-1}\Delta$ with $s_{w^*}=s$ (Lemma~\ref{lem:binormal}); this proves \eqref{eq:mixture}.

\emph{Strict monotonicity in $\alpha$.} With $A_{1}=\Phi(\tfrac s{\sqrt2})-\tfrac12$, $A_{2}=\Phi(\tfrac s{2\sqrt2})-\tfrac12$, the excess is $E(\alpha)=\alpha^{2}A_{1}+2\alpha(1-\alpha)A_{2}$ and $E'(\alpha)=2A_{2}+2\alpha(A_{1}-2A_{2})$. Since $x\mapsto\Phi(x)-\tfrac12$ is concave on $[0,\infty)$ and vanishes at $0$, it is subadditive along rays: $A_{1}=\Phi(2t)-\tfrac12\le2(\Phi(t)-\tfrac12)=2A_{2}$ with $t=\tfrac{s}{2\sqrt2}$ (strict for $s>0$); also $A_{1}>A_{2}$. Hence $E'(\alpha)\ge2A_{2}+2(A_{1}-2A_{2})=2(A_{1}-A_{2})>0$ for all $\alpha\in[0,1]$.

\emph{Asymptotics \eqref{eq:mixasymp}.} As $s\to0$: $A_{1}=\tfrac{s}{\sqrt2}\phi(0)+O(s^{3})$ and $A_{2}=\tfrac{s}{2\sqrt2}\phi(0)+O(s^{3})=A_{1}/2+O(s^{3})$, so $E(\alpha)=\alpha^{2}A_{1}+2\alpha(1-\alpha)A_{2}=\alpha A_{1}+O(s^{3})$; since $\phi(0)=1/\sqrt{2\pi}$, this is $\alpha s\,\phi(0)/\sqrt{2}+O(s^{3})=\frac{\alpha s}{2\sqrt{\pi}}+O(s^{3})$. As $s\to\infty$: $A_{1},A_{2}\to\tfrac12$, so $E(\alpha)\to\tfrac12(\alpha^{2}+2\alpha(1-\alpha))=\tfrac{\alpha(2-\alpha)}2$.

\emph{Remark~\ref{rem:bayes} (Bayes optimality of the linear score).} The class-conditional densities are $p_{\pm}(h)=\alpha\,\varphi_{\Sigma}(h-\mu_0\mp\Delta/2)+(1-\alpha)\,\varphi_{\Sigma}(h-\mu_{0})$ with $\varphi_\Sigma$ the $\mathcal N(0,\Sigma)$ density and $\mu_0$ the midpoint. Dividing numerator and denominator of $p_{+}/p_{-}$ by $\varphi_{\Sigma}(h-\mu_{0})$ and using $\varphi_{\Sigma}(h-\mu_0\mp\Delta/2)/\varphi_{\Sigma}(h-\mu_0)=\exp\big(\pm\tfrac12\Delta^{\top}\Sigma^{-1}(h-\mu_0)-\tfrac{s^{2}}{8}\big)$ gives
\[
\frac{p_{+}}{p_{-}}(h)=\frac{\alpha\,e^{v/2-s^{2}/8}+(1-\alpha)}{\alpha\,e^{-v/2-s^{2}/8}+(1-\alpha)},\qquad v=\Delta^{\top}\Sigma^{-1}(h-\mu_{0}),
\]
whose numerator is strictly increasing and denominator strictly decreasing in $v$ (for $\alpha,s>0$); hence the likelihood ratio is strictly increasing in the linear score $v$ (equivalently $w^{*\top}h$). Every point of the Bayes-optimal ROC is attained by thresholding $v$ (Neyman--Pearson), so the maximal AUC over all measurable classifiers equals the maximal linear-probe AUC, which is \eqref{eq:mixture}. The step dividing by the neutral density uses $\mu_{0}=\tfrac12(\mu_{+}+\mu_{-})$; for a neutral offset $b\neq0$ along the discriminant the denominator acquires a term $(1-\alpha)e^{b v'-b^{2}/2}$ of mixed monotonicity, and the conclusion is no longer automatic.

\emph{Corollary~\ref{cor:offset}.} With neutral mean offset $b$ (standardized, along $w$), the $\{m,u\}$ and $\{u,m\}$ cases have mean gaps $\tfrac{s_w}2-b$ and $\tfrac{s_w}2+b$; the displayed form follows. For monotonicity, write $M(b)=\Phi(\frac{s/2-b}{\sqrt2})+\Phi(\frac{s/2+b}{\sqrt2})$, so the excess is $E(\alpha)=\alpha^{2}A_{1}+\alpha(1-\alpha)(M-1)$. First, $M>1$ for \emph{every} $b$: $M>1$ iff $\Phi(\frac{s/2-b}{\sqrt2})>1-\Phi(\frac{s/2+b}{\sqrt2})=\Phi(-\frac{s/2+b}{\sqrt2})$, i.e.\ iff $s/2-b>-(s/2+b)$, which holds whenever $s>0$. Second, $M(b)\le M(0)=1+2A_{2}$, since $\partial M/\partial b=\frac{1}{\sqrt2}[\phi(\frac{s/2+b}{\sqrt2})-\phi(\frac{s/2-b}{\sqrt2})]<0$ for $b>0$ (and $M$ is even in $b$). Hence $E'(\alpha)=2\alpha A_{1}+(1-2\alpha)(M-1)$ is affine in $\alpha$ with $E'(0)=M-1>0$ and $E'(1)=2A_{1}-(M-1)\ge2(A_{1}-A_{2})>0$, so $E'>0$ on $[0,1]$ for every $b$.
\end{proof}

\subsection{Proof of Theorem~\ref{thm:detect}}
\begin{proof}
Under $H_{0}$ the scores are exchangeable across labels with continuous law, so $\widehat\auc$ is the (normalized) Mann--Whitney $U$-statistic with $\E=\tfrac12$ and exact variance $\sigma_{0}^{2}=\frac{n_{+}+n_{-}+1}{12n_{+}n_{-}}$ \cite{bamber1975area,mann1947test}; asymptotic normality holds as $\min(n_{+},n_{-})\to\infty$. The level-$\delta$ one-sided test rejects at $\widehat\auc>\tfrac12+z_{1-\delta}\sigma_{0}$. Under the alternative, $\widehat\auc$ concentrates around the population $\auc(\alpha)\le\aucmax(\alpha)$ with s.d.\ $\sigma_{A}=\sigma_{0}(1+O(\auc-\tfrac12))$ near the boundary regime; the normal power calculation gives power $\ge1-\gamma$ iff $\aucmax(\alpha)-\tfrac12\ge(z_{1-\delta}+z_{1-\gamma})\sigma_{0}(1+o(1))$, which is \eqref{eq:detcond}. Existence/uniqueness of $\astar$: $E(\alpha)=\aucmax(\alpha)-\frac12$ is continuous, strictly increasing (Theorem~\ref{thm:mixture}), $E(0)=0$ and $E(1)=A_{1}$; if the threshold is $\le A_1$ there is a unique crossing. For \eqref{eq:alphastar}: near the boundary $\alpha$ is small, so $E(\alpha)=2\alpha A_{2}(1+O(\alpha))$; for moderate $s$, $2A_{2}\approx\frac{s}{2\sqrt2}\cdot\frac{2}{\sqrt{2\pi}}=\frac{s}{2\sqrt\pi}$; with $n_{+}=n_{-}=n_{\mathrm{te}}/2$, $\sigma_{0}=\sqrt{\frac{n_{\mathrm{te}}+1}{3n_{\mathrm{te}}^{2}}}\approx\frac1{\sqrt{3n_{\mathrm{te}}}}$. Solving $\frac{\alpha s}{2\sqrt\pi}=(z_{1-\delta}+z_{1-\gamma})\frac1{\sqrt{3n_{\mathrm{te}}}}$ yields $\astar=\frac{2\sqrt\pi}{\sqrt3}\cdot\frac{z_{1-\delta}+z_{1-\gamma}}{s\sqrt{n_{\mathrm{te}}}}$, and $\frac{2\sqrt\pi}{\sqrt3}=2.046\ldots$. For $s\gtrsim6$, $A_{2}\approx\tfrac12$ saturates and the boundary becomes purity-limited, $\astar\approx(z_{1-\delta}+z_{1-\gamma})\sigma_{0}$, independent of $s$---the regime observed for all LLM combinations in E2.
\end{proof}

\section{Proof of Lemma~\ref{lem:beta} and the inversion formula}\label{app:beta}
\begin{proof}
Pooled over labels and marked/neutral status at purity $\alpha$, the mean of the component means is $\mu_{0}$ and their covariance is $\E[mm^{\top}]$ where $m=\pm\Delta/2$ w.p.\ $\alpha/2$ each and $0$ w.p.\ $1-\alpha$: $\E[mm^{\top}]=\frac{\alpha}{4}\Delta\Delta^{\top}$. With isotropic within-component covariance $\sigma^{2}I_{d}$, the pooled covariance is $\Sigma_{\mathrm{pool}}=\sigma^{2}I_{d}+\frac{\alpha q}{4}\,\hat\Delta\hat\Delta^{\top}$, $q=\|\Delta\|^{2}$, with spectrum $\{\sigma^{2}+\frac{\alpha q}4,\ \sigma^{2}\ (\times\,d{-}1)\}$. Then $\tr=d\sigma^{2}+\frac{\alpha q}4$, $\bar\lambda=\sigma^{2}+\frac{\alpha q}{4d}$, and
\[
\betaB=\frac{\lambda_{1}-\bar\lambda}{\tr}
=\frac{\frac{\alpha q}{4}\big(1-\frac1d\big)}{d\sigma^{2}+\frac{\alpha q}{4}},
\]
which is $0$ at $q=0$, strictly increasing in $\alpha q$ (derivative of $x\mapsto\frac{cx}{a+x}$), and $\to1-\frac1d$ as $\alpha q\to\infty$. The range claim holds for arbitrary spectra: $\lambda_{1}\le\tr$ and $\bar\lambda=\tr/d$ give $\betaB\le\frac{\tr-\tr/d}{\tr}=1-\frac1d$, while $\lambda_{1}\ge\bar\lambda$ gives $\betaB\ge0$. Solving the display for $q/\sigma^{2}$ gives the inversion \eqref{eq:betainv}: $s^{2}=q/\sigma^{2}=\frac{4d\betaB}{\alpha(1-\frac1d-\betaB)}$ for $\betaB<1-\frac1d$. Under the curved model of Theorem~\ref{thm:curvature}, the ambient mean gap is the chord $c(r)=g(\kap,r)\,r$, so $q=c(r)^{2}$ is strictly decreasing in $\kap$ at fixed $r$ (since $g$ is), and therefore so is $\betaB$.
\end{proof}

\section{Experimental details}\label{app:exp}

\subsection{Prompt grid}
Context cells: $44$ professions $\times$ $5$ frames (e.g., ``The \{M\} works as a \{C\}.'', ``As a \{C\}, the \{M\} is highly respected.''), $20$ traits $\times$ $3$ frames, $20$ activities $\times$ $3$ frames; $340$ cells total. Marker pairs (positive/negative group): gender \emph{man/woman, boy/girl, father/mother, brother/sister, gentleman/lady, husband/wife, uncle/aunt, grandfather/grandmother}; race \emph{White X/Black X} for $X\in\{$man, woman, person, American, student, worker, teenager, customer$\}$; religion \emph{Christian X/Muslim X} for $X\in\{$man, woman, person, family man, student, neighbor, colleague, friend$\}$; age \emph{young X/elderly X} for $X\in\{$man, woman, person, employee, volunteer, neighbor, customer, voter$\}$. Marker pairs rotate deterministically over cells: $340$ minimal pairs per dimension. Neutral heads: \emph{person}, \emph{individual} over the same cells ($680$ neutral sentences per dimension). Total $5{,}440$ prompts.

\subsection{Models}
GPT-2 S/M/L/XL and BERT-base-uncased from locally mirrored checkpoints; Qwen2.5-7B-Instruct from a local snapshot; all loaded with \texttt{transformers} \cite{wolf2020transformers} in evaluation mode, bfloat16 for the 7B model. Features: mean over non-pad token hidden states; layers at fractional depths $\{0.25,0.5,0.75,1.0\}$ for all models plus every layer for GPT-2 Small and BERT-base. Mid-depth layer used unless stated.

\subsection{Probing and testing}
Logistic regression (scikit-learn, \texttt{lbfgs}, $C{=}1$, \texttt{max\_iter}=3000) on standardized features; grouped 5-fold CV (groups = context cells; fold assignment reshuffled per seed) $\times$ 5 seeds; counterfactual pair members share a group by construction. Detection: the seed-averaged pooled out-of-fold AUC is tested one-sided against the \emph{calibrated} null scale $1.4\,\sigma_{0}$ of Section~\ref{sec:stats}; E9 measures the realized level of the uncalibrated variants ($\sigma_0$ and the earlier $\sigma_0/\sqrt5$) and of the calibrated test under the full protocol, and the multiplier $1.4$ covers the largest measured null inflation ($1.31\times$ on pool-respecting cells) with margin. Released JSONs store the calibrated $p$ together with the AUC and $n$ from which the raw $z$ (vs $\sigma_{0}$) follows deterministically. Released p-values use the numerically stable survival function; an earlier build computed $1-\Phi(z)$, which underflows to exactly $0$ for $z\gtrsim8$ in double precision, and the release notes flag the affected historical files.

\subsection{Nonlinear probes (E10)}
RBF-kernel SVM (\texttt{SVC}, $C{=}1$, $\gamma=$\texttt{scale}, decision-function scores) and a one-hidden-layer MLP ($64$ ReLU units, Adam, early stopping), both on standardized features under the identical grouped-CV protocol and detection test as the linear probe. On the synthetic sphere generator of E3-B, a geodesic oracle scores by the difference of geodesic distances to the two true centers after projecting onto the sphere; it upper-bounds what curvature-aware detection can achieve.

\subsection{Neutral-leakage protocol (E11)}
Stereotype labels follow the WinoBias occupation lists \cite{zhao2018winobias} (BLS-derived): male-dominated $=$ \emph{carpenter, manager, lawyer, farmer, salesperson, mechanic, doctor (physician), programmer (developer), security guard (guard)}; female-dominated $=$ \emph{cashier, teacher, nurse, receptionist, designer, accountant, librarian}; parenthesized names give the WinoBias occupation matched to our inventory. These $16$ professions span $80$ context cells and $160$ neutral sentences for the gender dimension. Part C draws neutral texts from this subset only and assigns each a label equal to its stereotype with probability $\pi$ and uniform otherwise, so $\pi=\tfrac12$ reproduces the main protocol and $\pi=1$ is fully stereotype-aligned. 

\subsection{Naturalistic paraphrase subset}
Paraphrases generated by a locally served instruction model via \texttt{ollama} with an instruction that pins the exact marker phrase; a pair is kept only if both paraphrases contain their full marker string and neither contains the counterpart's ($208$ pairs retained). This subset is a robustness probe, not a benchmark.

\subsection{Algorithm}
\begin{algorithm}[!t]
\caption{Powered bias audit (practical procedure)}
\label{alg:audit}
\begin{algorithmic}[1]
\REQUIRE audit corpus with measured purity $\hat\alpha$; level $\delta$; power $1-\gamma$; minimal effect $s_{\min}$
\STATE $n_{\mathrm{te}} \leftarrow$ smallest $n$ with $\aucmax(\hat\alpha; s_{\min})-\tfrac12 \ge (z_{1-\delta}+z_{1-\gamma})\,\sigma_{0}(n)$ \hfill(Thms.~\ref{thm:mixture},~\ref{thm:detect})
\IF{corpus smaller than $n_{\mathrm{te}}$}
    \STATE \textbf{report} ``underpowered for $s_{\min}$ at $\hat\alpha$''; optionally raise $\alpha$ by marker-enrichment and recompute
\ENDIF
\STATE extract features; train probe with grouped CV; pooled OOF $\widehat{\auc}$
\STATE \textbf{if} $\widehat{\auc}>\tfrac12+z_{1-\delta}\,c\,\sigma_{0}(n_{\mathrm{te}})$ (with $c{=}1$ for a single held-out split, $c{=}1.4$ for the pooled-OOF CV statistic; E9): report detection with effect size $\hat s$ (cross-fitted) \textbf{else}: report a \emph{powered} null (no linear signal of size $\ge s_{\min}$ at purity $\hat\alpha$)
\end{algorithmic}
\end{algorithm}

\subsection{Compute}
Feature extraction: one forward pass per model over $5{,}440$ prompts on a single RTX~4090 ($<1$~min for GPT-2 family and BERT; $\approx3$~min for Qwen2.5-7B in bf16). All probing and analysis are CPU. End-to-end: $\approx10$ GPU-minutes $+$ several CPU-hours across the reported grids. Seeds fixed and recorded in each output JSON.

\end{document}